%% file: alft.tex
\documentclass[11pt]{article}
\usepackage[letterpaper, margin=1in]{geometry}

\usepackage[utf8]{inputenc} 
\usepackage[T1]{fontenc}    
\usepackage{hyperref}       
\usepackage{url}            
\usepackage{booktabs}       
\usepackage{amsfonts}       
\usepackage{nicefrac}       
\usepackage{microtype}      
\usepackage{xcolor}         
\usepackage{graphicx}

\usepackage{algorithm}
\usepackage{algorithmic}

\usepackage{amsmath}
\usepackage{amssymb}
\usepackage{mathtools}
\usepackage{amsthm}

\input{macros}

\title{Distributional Alignment Games\\ for Answer-Level Fine-Tuning}
\author{Mehryar Mohri\\
Google Research\\
\texttt{mohri@google.com}
\and
Jon Schneider\\
Google Research\\
\texttt{jschnei@google.com}
\and
Yifan Wu\thanks{Work done when Yifan Wu was a student researcher with Google Research. }\\
Microsoft Research\\
\texttt{yifan.wu2357@gmail.com}
}
\date{}

\begin{document}

\maketitle
\begin{abstract}
\input{abstract}
\end{abstract}

\tableofcontents
\clearpage

\input{introduction}

\input{framework}

\input{generic_alg}

\input{diversity}

\input{self-improvement}

\input{pairwise}
\input{safety}

\input{experiments}

\input{conclusion}

\newpage
\bibliography{alft}
\bibliographystyle{abbrvnat}

\newpage
\appendix

\input{appendix}

\end{document}

%% file: macros.tex
\usepackage[capitalize,noabbrev]{cleveref}
\usepackage{natbib}

\usepackage{dsfont}
\usepackage[mathscr]{euscript}
\usepackage{expl3}
\usepackage{xcolor}
\usepackage{colortbl}
\usepackage{thmtools}
\usepackage{thm-restate}
\usepackage{mathtools}
\usepackage{multirow}
\usepackage{wrapfig}
\usepackage{lipsum}
\usepackage{accents}

\usepackage{pifont}
\DeclareMathAlphabet\mathbb{U}{msb}{m}{n}
\usepackage{xpatch}

\def\Rset{\mathbb{R}}

\DeclareMathOperator*{\E}{\mathbb E}
\DeclareMathOperator*{\argmax}{argmax}
\DeclareMathOperator*{\argmin}{argmin}

\DeclareMathOperator*{\Span}{span}

\DeclareMathOperator{\Improv}{Improv}
\DeclareMathOperator{\Ind}{\mathbb{I}} 
\DeclareMathOperator{\Int}{int}

\DeclarePairedDelimiter{\bracket}{[}{]}
\DeclarePairedDelimiter{\curl}{\{}{\}}
\DeclarePairedDelimiter{\paren}{(}{)}
\DeclarePairedDelimiter{\norm}{\|}{\|}
\DeclarePairedDelimiter{\tri}{\langle}{\rangle}

\ExplSyntaxOn
\tl_const:Nn \c_my_uc_alphabet_tl { ABCDEFGHIJKLMNOPQRSTUVWXYZ }
\tl_const:Nn \c_my_full_alphabet_tl { ABCDEFGHIJKLMNOPQRSTUVWXYZ
  abcdefghijklmnopqrstuvwxyz }

\tl_map_inline:Nn \c_my_uc_alphabet_tl
 { \cs_gset:cpn { c#1 } { \mathcal{#1} } }

\tl_map_inline:Nn \c_my_uc_alphabet_tl
 { \cs_gset:cpn { s#1 } { \mathscr{#1} } }

\tl_map_inline:Nn \c_my_full_alphabet_tl
 {
  \str_if_eq:nnF { #1 } { f }
  {
    \cs_gset:cpn { b#1 } { \mathbf{#1} }
  }
  
  \cs_gset:cpn { sf#1 } { \mathsf{#1} }
 }
\ExplSyntaxOff

\newcommand{\h}{\widehat}
\newcommand{\ov}{\overline}
\newcommand{\wt}{\widetilde}
\newcommand{\e}{\epsilon}
\newcommand{\ignore}[1]{}

\newlength{\dhatheight}

\newcommand{\KL}{\sfD_{\mathrm{KL}}}

\newcommand{\EX}{\sfE}
\newcommand{\Inc}{\text{Incoherence}}
\newcommand{\algname}[1]{\textsc{#1}}

\hypersetup{
  breaklinks   = true, 
  colorlinks   = true, 
  urlcolor     = blue, 
  linkcolor    = blue, 
  citecolor   = blue 
}

\declaretheorem{theorem}
 
\newtheorem{proposition}[theorem]{Proposition}

\newtheorem{definition}[theorem]{Definition}

%% file: abstract.tex
   We focus on the problem of \emph{Answer-Level Fine-Tuning} (ALFT),
   where the goal is to optimize a language model based on the
   correctness or properties of its final answers, rather than the
   specific reasoning traces used to produce them. Directly optimizing
   answer-level objectives is computationally intractable due to the
   need to marginalize over the vast space of latent reasoning
   paths. To overcome this, we propose a general game-theoretical
   framework that lifts the problem to a \emph{Distributional
     Alignment Game}. We formulate ALFT as a two-player game between a
   Policy (the generator) and a Target (an auxiliary distribution). We
   prove that the Nash Equilibrium of this game corresponds exactly to
   the solution of the original answer-level optimization
   problem. This variational perspective transforms the intractable
   marginalization problem into a tractable projection problem. We
   demonstrate that this framework unifies recent approaches to
   diversity and self-improvement (coherence) and provide efficient
   algorithms compatible with Group Relative Policy Optimization
   (GRPO), such as \emph{\algname{Coherence-GRPO}}, yielding
   significant complexity gains in mathematical reasoning tasks.

%% file: introduction.tex
\section{Introduction}

In reasoning-intensive domains such as mathematics, code generation,
and open-ended question answering, the primary utility of a Large
Language Model (LLM) lies in the correctness of its final
\emph{answer} $z$, rather than the specific phrasing of the
intermediate \emph{reasoning trace} $y$.  This distinction has driven
a surge of interest in Chain-of-Thought (CoT) prompting
\citep{Wei2022} and reasoning-focused alignment strategies.  While
``Process Supervision'' aims to guide the model step-by-step
\citep{Uesato2022, Lightman2023}, ``Outcome Supervision'' focuses purely
on the final result, allowing the model the flexibility to discover
diverse valid reasoning paths \citep{Wang2022}. This
latter setting defines the problem of \emph{Answer-Level Fine-Tuning
  (ALFT)}. An ideal reasoning policy should be robust yet
flexible: free to explore various trajectories provided they converge
to correct or consistent outcomes.

However, ALFT presents a fundamental computational bottleneck that is
absent in standard trace-level fine-tuning.  Standard alignment
methods, such as Supervised Fine-Tuning (SFT) or Direct Preference
Optimization (DPO) \citep{Rafailov2023}, optimize the likelihood of a
specific, observed trace $y$.  In contrast, ALFT requires optimizing
the \emph{marginal probability} of the answer $z$.  Because the
mapping from reasoning trace to answer $z = \EX(y)$ is often
non-differentiable (e.g., executing a program) and many-to-one,
computing the gradient of the marginal likelihood requires
marginalizing over the vast combinatorial space of all latent traces
consistent with $z$.  This makes direct optimization intractable and
standard gradient estimators (like REINFORCE) prohibitively
high-variance, often requiring sophisticated variance reduction
techniques or prohibitively large sample sizes to stabilize training.

To overcome this, we propose a fundamental shift in
perspective. Instead of attacking the marginalization sum directly, we
lift the optimization problem to a \emph{Distributional Alignment
  Game}.  We introduce an auxiliary \emph{Target Distribution} $\sfq$,
which serves as a variational proxy for the intractable marginals.  By
leveraging Fenchel duality, we formulate ALFT as a two-player game
between a \emph{Policy} $\pi$ (the generator) and a \emph{Target}
$\sfq$.  The Policy attempts to minimize the divergence of its trace
distribution from the Target, while the Target adapts to satisfy
high-level distributional properties—such as enforcing consensus
(coherence), maximizing coverage (diversity), or satisfying safety
constraints—thereby guiding the Policy.  This variational perspective
transforms the hard marginalization problem into a tractable
projection problem, where the difficulty is offloaded to the
\emph{Target Step}: finding the optimal $\sfq$ that best challenges or
guides the current policy.

Crucially, this framework provides a unified theoretical lens for
understanding recent heuristic successes in reasoning alignment.  We
demonstrate that distinct alignment goals are merely different
instantiations of the game's Target Step.  For instance, we show that
the \emph{diversity-promoting} objective of maximizing entropy
corresponds to a game against an adversarial target, where the Nash
Equilibrium justifies the ``Inverse-Frequency Reward'' heuristic
recently proposed by \citet{Li2025}.  Conversely, the
\emph{coherence-promoting} objective (Self-Improvement) corresponds to
a cooperative game, where the equilibrium is the Bregman Centroid of
the generated distributions.  This rigor allows us to replace ad-hoc
reward engineering with derived solutions to well-posed optimization
problems.

We further bridge the gap between theory and practice by deriving a
scalable algorithmic framework.  We show that the policy update steps
in our game are compatible with Group Relative Policy Optimization
(GRPO) \citep{Shao2024}, an efficient policy gradient method that
operates on groups of outputs. By deriving rewards directly from the
optimal Target distribution $\sfq^*$, we instantiate algorithms like
\emph{\algname{Coherence-GRPO}} that solve the alignment game without explicit
marginalization.

Our main contributions are as follows:
\begin{enumerate}
\item \textbf{General Framework:} We formalize ALFT as a min-max game
  that is convex-concave in the dual variables. We prove
  \emph{consistency}: the Nash Equilibrium of this game recovers the
  exact optimal solution to the original answer-level problem,
  identifying the optimal policy as a projection of the reference onto
  the target-compatible traces.
    
    \item \textbf{Unification of Objectives:} We provide a single
      mathematical umbrella for disparate alignment goals. We prove that
      maximizing diversity implies an inverse-frequency update, while
      minimizing incoherence implies a consensus-based update, grounding
      recent empirical methods in convex duality theory.

    \item \textbf{Scalable Algorithms:} We propose game-theoretic variants
      of GRPO. We introduce \emph{\algname{Coherence-GRPO}} for discrete domains 
      and \emph{\algname{Pairwise-GRPO}} for open-ended domains, demonstrating 
      how to compute rigorous advantage signals from group statistics.

    \item \textbf{Extensions to Safety and Continuous Domains:} We extend
      the framework beyond standard exact matching. We introduce a
      pairwise formulation for open-ended generation where exact answers
      are undefined, and a Primal-Dual algorithm for enforcing
      distributional constraints (e.g., safety and fairness) by
      interpreting Lagrange multipliers as dynamic penalty weights.
\item \textbf{Experiment: Self-Improvement via Coherence} We conduct experiments on self-improvement of language models without ground truth. We test algorithms \emph{\algname{Coherence-GRPO}} and \emph{\algname{Pairwise-GRPO}} with three models, \texttt{\small Qwen2.5-3B-Instruct}, \texttt{\small Phi-3-mini-4k-instruct}, and \texttt{\small Llama-3.2-3B\-Instruct}, on TriviaQA (a question-answering dataset) and GSM8K (a simple math dataset). Across three 3B-class instruction-tuned LLMs, our algorithm consistently improves performance: on GSM8K, \algname{Pairwise-GRPO} and \algname{Coherence-GRPO} increase accuracy by $+3.18$ to $+9.18$ percentage points ($+4.80\%$ to $+12.46\%$ relative). On TriviaQA, \algname{Pairwise-GRPO} yields substantial gains (up to $+42.06\%$ in relative EM and $+18.12\%$ in relative F1), while \algname{Coherence-GRPO} provides smaller and less consistent improvements. 
\end{enumerate}

%% file: framework.tex
\section{General Framework: Alignment as a Distributional Game}

\subsection{Problem: Answer-Level Fine-Tuning}

Let $\sX$ be the space of input prompts, $\sY$ the space of reasoning traces, and
$\sZ$ the space of final answers. A deterministic \emph{extraction
  function} $\EX \colon \sY \to \sZ$ maps traces to answers. A policy
$\pi(y | x)$ (a conditional distribution over traces) induces a
marginal distribution over answers:
\[
\nu_\pi(z|x) = \sum_{y \in \EX^{-1}(z)} \pi(y|x).
\]
Our goal is to find a policy $\pi$ that minimizes a loss $\cL$
that depends only on this answer distribution $\nu_\pi$, subject to a
regularization constraint that keeps the policy close to a reference
$\pi_0$. We formulate this as the following convex optimization
problem over a convex set $\Pi \subseteq \Delta(\sY)^\sX$:
\begin{equation}
\label{eq:primal}
\mspace{-10mu} \min_{\pi \in \Pi} \cJ(\pi)
= \E_{x} \bracket[\big]{ \cR(\nu_\pi(\cdot|x))
  + \beta \, \KL(\pi(\cdot|x) \| \pi_0(\cdot|x)) }, \mspace{-7mu}
\end{equation}
where the functional $\cR\colon \Delta(\sZ) \to \Rset$ encodes our
answer-level goal (e.g., minimizing entropy for coherence, or
maximizing entropy for diversity) and where $\beta > 0$ is a
hyperparameter. We assume $\cR$ is convex and lower semi-continuous
(l.s.c.) to ensure the existence of solutions and applicability of
duality. The overall objective $\cJ(\pi)$ is strictly convex due to
the KL term, guaranteeing a unique optimal policy.

\textbf{Why is this hard?}
Gradient-based optimization of \eqref{eq:primal} is challenging
because the objective $\cR$ acts on the marginal $\nu_\pi$. Computing
the gradient requires backpropagating through the summation
$\nu_\pi(z) = \sum_{y \in \EX^{-1}(z)} \pi(y)$. Since the set
$\EX^{-1}(z)$ is unknown and vast, we cannot compute this sum or its
gradient efficiently. Monte-Carlo estimates suffer from extreme
variance because the ``credit'' for a good answer $z$ must be assigned
to the specific trace $y$ that produced it, without knowing if other
traces would have done better.

\subsection{Equivalent Formulation as Min-Max Game via Fenchel Duality}

We now show how to transform the intractable marginalization problem
into a tractable game. We rely on the Fenchel Duality Theorem
\citep{Rockafellar1996}.

Let $\cR^*\colon \Rset^{|\sZ|} \to \Rset$ denote the Fenchel
conjugate of $\cR$, defined as:
\[
  \cR^*(u)
  = \sup_{\nu \in \Delta(\sZ)} \curl*{\tri*{\nu, u} - \cR(\nu)}.
\]
Since $\cR$ is convex and lower semi-continuous, the Fenchel-Moreau
theorem \citep{Rockafellar1996} guarantees that $\cR$ equals its
biconjugate $\cR^{**}$:
\begin{equation}
\label{eq:fenchel_var}
\cR(\nu)
= \sup_{u \in \Rset^{|\sZ|}} \curl*{ \langle \nu, u \rangle - \cR^*(u) }.
\end{equation}
This variational form allows us to introduce an auxiliary variable $u$
to decouple the marginalization. Substituting \eqref{eq:fenchel_var}
into the primal objective \eqref{eq:primal}:
\begin{multline*}
  \min_{\pi \in \Pi} \E_{x \in \Rset^{|\sZ|}} \Biggl[
    \sup_{u} \paren*{ \tri*{\nu_\pi(\cdot | x), u } - \cR^*(u)} \\
    + \beta \, \KL(\pi(\cdot | x) \parallel \pi_0(\cdot | x)) \Biggr].
\end{multline*}
The game is naturally defined over the dual variables
$u \in \Rset^{|\sZ|}$.  However, the objective is invariant under
constant shifts.  First, observe that for any functional $\cR$ defined
on the probability simplex $\Delta(\sZ)$, the Fenchel conjugate
satisfies the shift property $\cR^*(u + c\mathbf{1}) = \cR^*(u) + c$.
This holds because $\nu \in \Delta(\sZ)$ implies $\sum \nu(z) = 1$,
so:
\begin{align*}
  \cR^*(u + c\mathbf{1})
  &= \sup_{\nu \in \Delta(\sZ)} \curl*{ \tri*{\nu, u + c\mathbf{1}} - \cR(\nu) } \\
  &= \sup_{\nu \in \Delta(\sZ)} \curl*{ \tri*{\nu, u} - \cR(\nu)
     + c \sum_{z}\nu(z)} \\
  &= \cR^*(u) + c.
\end{align*}
Thus, the dual objective term $\tri{\nu_\pi, u} - \cR^*(u)$ is
invariant under the transformation $u \leftarrow u + c\mathbf{1}$, as
the linear shift $c$ cancels exactly with the conjugate shift.  This
redundancy implies that the effective dual space is the quotient space
$\Rset^{|\sZ|} / \Span(\mathbf{1})$.  We can therefore parameterize
the dual variables without loss of generality by choosing a canonical
representative from each equivalence class.  Specifically, we use the
parameterization $u(z) = -\beta \log \sfq(z)$ with
$\sfq \in \Int(\Delta(\sZ))$.  This mapping establishes a bijection
between the quotient space and the interior of the simplex.  For any
arbitrary vector $u \in \Rset^{|\sZ|}$, there exists a unique scalar
shift $c$ that satisfies the probability normalization constraint
$\sum_{z} \exp(-(u(z) + c)/\beta) = 1$.  Solving for $c$ reveals that
this unique shift is determined exactly by the Log-Sum-Exp function
$c = \beta \log \paren*{ \sum_{z \in \sZ} \exp(-u(z)/\beta)}$.

By setting the canonical representative to $v = u + c\mathbf{1}$, we
guarantee that $\sfq(z) = \exp(-v(z)/\beta)$ is a valid normalized
distribution.  We define the transformed conjugate functional
$\Psi(\sfq) = \cR^*(-\beta \log \sfq)$.
The coupling term becomes:
\begin{align*}
  \tri*{ \nu_\pi(\cdot | x), -\beta \log \sfq }
  & = -\beta \sum_{z \in \sZ} \nu_\pi(z | x) \log \sfq(z) \\
  & = -\beta \sum_{z \in \sZ} \bracket*{ \sum_{y \in \EX^{-1}(z)} \pi(y | x) } \log \sfq(z) \\
  & = -\beta \sum_{z \in \sZ} \sum_{y \in \EX^{-1}(z)} \pi(y | x) \log \sfq(\EX(y))
  \tag{Since $\EX(y)=z$ for $y \in \EX^{-1}(z)$} \\
  & = -\beta \sum_{y \in \sY} \pi(y | x) \log \sfq(\EX(y))
  \tag{Partition property: $\sY = \bigcup_z \EX^{-1}(z)$} \\
  & = -\beta \E_{y \sim \pi(\cdot | x)} [\log \sfq(\EX(y))].
\end{align*}
Substituting this back, we obtain the following min-max objective
function $\cG(\pi, \sfq)$:
\begin{equation}
\label{eq:game_def}
\begin{split}
\cG(\pi, \sfq) \triangleq \E_{x} \bigg[ & \beta \KL(\pi(\cdot|x) \parallel \pi_0(\cdot|x)) \\
& - \beta \E_{y \sim \pi(\cdot | x)}[\log \sfq(\EX(y))] - \Psi(\sfq) \bigg].
\end{split}
\end{equation}

\subsection{Theoretical Consistency}

The derivation above leads directly to our main result: the
equilibrium of the game defined by $\cG$ recovers the solution
to the original hard answer-level problem.

\begin{theorem}[Consistency of the Game]\label{thm:consistency}
  Let $\cR\colon \Delta(\sZ) \to \Rset$ be a convex, lower
  semi-continuous functional. Then:
\begin{enumerate}
\item The primal problem \eqref{eq:primal} is equivalent to the
  min-max game:
  \[
    \min_{\pi \in \Pi} \cJ(\pi)
    = \min_{\pi \in \Pi} \max_{\sfq \in \Delta(\sZ)} \cG(\pi, \sfq).
  \]

\item For a fixed Target distribution $\sfq$, the optimal policy
  $\pi^*$ is the unique projection of $\pi_0$ onto the traces
  compatible with $\sfq$:
  \[
    \pi^*(y|x) \propto \pi_0(y|x) \sfq(\EX(y)|x).
  \]
\end{enumerate}
\end{theorem}

We further show that finding an approximate equilibrium of the game
yields a good solution to the original problem.

\begin{proposition}[Approximation Guarantee]
  Let $(\h \pi, \h \sfq)$ be an $\e$-approximate
  equilibrium of the game $\cG$, such that
  $\cG(\h \pi, \h \sfq) \le \min_\pi
  \max_\sfq \cG(\pi, \sfq) + \e$. Then
  $\h \pi$ is an $\e$-approximate minimizer of the primal
  objective $\cJ(\pi)$.
\end{proposition}
\begin{proof}
  By Part 1 of the Theorem,
  $\cJ(\pi) = \max_\sfq \cG(\pi,
  \sfq)$. If $(\h \pi, \h \sfq)$ is an
  $\e$-equilibrium, then
  $\cJ(\h \pi) \le \cJ(\pi^*) + \e$.
\end{proof}

Finding an equilibrium of this minimax game ends up being a (practically, and in many specific cases, theoretically) far more tractable problem than the original primal minimization problem. In Appendix \ref{app:solving-minimax}, we describe approaches for this computation via no-regret and best-response dynamics. 

\subsection{Specific Instantiations}

The power of our framework lies in its universality. By selecting the
functional $\cR$, we recover distinct alignment paradigms as specific
instantiations of the game's Target Step:

\textbf{Standard RL.} When the goal is simply to maximize a scalar reward $r(z)$ (e.g., $r(z) = \Ind[z=z_{\text{gold}}]$), the answer-level functional is linear: $\cR(\nu) = - \E_{z \sim \nu}[r(z)]$. In this case, the convex conjugate $\cR^*(u)$ collapses the dual space to a \emph{single point}, fixing the optimal Target distribution to the exponentiated reward distribution: $q^*(z) \propto \exp(r(z)/\beta)$ (see Proposition~\ref{prop:dpo_duality}).

\textbf{Supervised Fine-Tuning.} One interpretation of SFT is that we want to train the model to match a specific ground truth Dirac distribution $\delta_{\text{ground truth}}$. This is captured by the functional $\cR(\nu) = \KL(\delta_{\text{ground truth}} \parallel \nu)$. The corresponding game is trivial; the Target is fixed to the
ground truth $q^*(z) = \delta_{\text{ground truth}}(z)$.

\textbf{Diversity.}  When the goal is
exploration, we choose $\cR(\nu) = -H(\nu)$ (negative entropy). The
Target $\sfq$ acts as an \emph{adversary}, maximizing the dual
potential $\Psi(\sfq)$ (the log-partition function). The optimal
strategy $\sfq^*$ assigns high mass to under-represented answers,
forcing the policy to spread its probability mass. As we show in
Appendix~\ref{sec:diversity}, this recovers the ``Inverse-Frequency''
heuristic \citep{Li2025}.

\textbf{Safety.} The goal of safety can be cast as requiring that the policy $\pi$ lies in a ``safe'' set $\cC$. This can be enforced by the functional $\cR(\nu) = \mathrm{Ind}_{\cC}(\nu)$ (taking value $0$ if $\nu \in \cC$ and $+\infty$ otherwise). For this $\cR$, the corresponding target $\sfq^*
= \argmin_{\sfq \in \cC} \KL(\sfq \parallel \nu_\pi)$, the information projection of $\sfq$ onto $\cC$. See Appendix~\ref{sec:safety} for details. 

\textbf{Coherence.}  Finally, we remark briefly here that (with minor changes to the setup) the goal of self-improvement via consensus \citep{MohriSchneiderWu2025} can be captured by setting $\cR$ to a specific expected divergence. We discuss this in greater detail in Section~\ref{sec:coherence}.

%% file: generic_alg.tex
\section{Generic Algorithm:  Solving the Game via GRPO}
\label{sec:generic}

We now describe practical algorithms for solving the Distributional
Alignment Game. We use \emph{Group Relative Policy Optimization
  (GRPO)} as the solver for the Policy Step, feeding it rewards
derived from the Target Step.

\emph{Group Relative Policy Optimization (GRPO)} is a policy gradient
algorithm designed for scenarios where evaluation occurs over a
\emph{group} of outputs. For each input $x$, GRPO samples a group of
$K$ outputs $G = \{y_1, \dots, y_K\}$ generated by the current policy
$\pi_{\text{old}}$. It updates the policy using a surrogate objective
that relies on advantages $A_i$ computed relative to the group
average:
\begin{align}
\label{eq:grpo}
& \mspace{-10mu} \cL_{\text{GRPO}}(\pi) \nonumber\\
& \mspace{-10mu} = \E_{x} \bracket*{ \E_{G \sim \pi_{\text{old}}} \bracket*{ \frac{1}{K} \sum_{i=1}^K
  \frac{\pi(y_i|x)}{\pi_{\text{old}}(y_i|x)} A_i
  - \beta \KL(\pi \| \pi_{\text{ref}}) }}
\end{align}
Crucially, GRPO is an optimization engine that requires an external
definition of \emph{advantage} or \emph{reward}. In our framework, we
define a general algorithm by deriving these advantages directly from
the game's equilibrium condition.

Recall that the optimal ALFT policy satisfies
$\pi^*(y|x) \propto \pi_0(y|x)\sfq^*(\EX(y)|x)$ (Theorem~\ref{thm:consistency}). This
suggests that the raw \emph{reward} for a trace $y$ should be determined by
the likelihood of its answer under the optimal Target distribution
$\sfq^*$. We therefore define the reward for the $i$-th trace in the group
as:
\[
  R_i = \log \sfq^*(\EX(y_i)|x).
\]
The advantage $A_i$ is then defined as the standardized reward
relative to the group statistics, which serves as a baseline to reduce
variance:
\begin{equation}
\label{eq:advantage}
A_i
= \frac{R_i - \text{Mean}(\{R_1, \dots, R_K\})}{\text{StdDev}(\{R_1, \dots, R_K\})
  + \e} \ ,
\end{equation}
where $\e > 0$ is a small positive number.  An advantage $A_i > 0$
indicates that trace $y_i$ aligns better with the Target than the
average trace in the group. This general algorithm can be applied to
any distributional alignment problem where the optimal Target $\sfq^*$
can be estimated. We now present two specific instantiations of this
algorithm.

\subsection{Theoretical Justification for GRPO}

To understand why GRPO solves the game, consider the \textbf{Policy
  Step} derived in Theorem 2. For a fixed target $\sfq^*$, the
optimal policy $\pi^*$ must minimize the divergence to the lifted
target
$\tilde{\sfq}(y|x) \propto \pi_0(y|x)\sfq^*(\EX(y)|x)$.
This minimization problem can be rewritten as maximizing an expected
reward subject to a KL constraint:
\begin{align*}
& \argmin_\pi \KL(\pi \| \tilde{\sfq})\\
& = \argmin_\pi \E_{y \sim \pi} \bracket*{ \log \frac{\pi(y)}{\pi_0(y)\sfq^*(\EX(y))} } \\
& = \argmax_\pi \E_{y \sim \pi} \bracket*{ \underbrace{\log \sfq^*(\EX(y))}_{\text{Reward } R(y)} - \underbrace{\log \frac{\pi(y)}{\pi_0(y)}}_{\text{KL Penalty}} },
\end{align*}
an Entropy-Regularized Reinforcement Learning problem where the reward
function is defined by the Target distribution:
\begin{equation}
\label{eq:game_reward}
R(y) = \beta \log \sfq^*(\EX(y)).
\end{equation}
GRPO is an efficient gradient estimator for exactly this type of
objective. It estimates the policy gradient using a group of samples
$G=\{y_1, \dots, y_K\}$ and a group-based baseline to reduce variance.
Crucially, the Advantage formulation in GRPO (Eq.~\ref{eq:advantage})
can be viewed as an empirical estimate of the gradient direction
derived from the Nash Equilibrium condition.

This connection to entropy-regularized RL suggests a deeper
relationship with recent advancements in preference optimization. We
formalize this link in the following proposition, which establishes
our framework as the dual of Direct Preference Optimization (DPO).

\begin{proposition}[Duality with Direct Preference Optimization]
\label{prop:dpo_duality}
The Distributional Alignment Game is the mathematical dual of Direct
Preference Optimization (DPO) for answer-level
objectives. Specifically, the optimal target distribution $\sfq^*$
corresponds to the exponentiated ground-truth reward function $r^*$ of
the equivalent RL problem:
\begin{equation}
\label{eq:dpo_duality}
\sfq^*(\EX(y)) = \frac{1}{Z} \exp\paren*{ \frac{r^*(y)}{\beta} },
\end{equation}
where $Z$ is the partition function. Thus, while DPO eliminates the
reward model to solve for the policy directly, our framework
eliminates the policy to solve for the optimal target distribution.
\end{proposition}

\begin{proof}
  In the standard RL setting used by DPO \citep{Rafailov2023}, the
  optimal policy $\pi^*$ maximizing a reward $r^*$ subject to
  KL-regularization from $\pi_0$ has the closed form:
\[
\pi^*(y) \propto \pi_0(y) \exp\paren*{ \frac{r^*(y)}{\beta} }.
\]
In our game-theoretic framework (Theorem 1), we derived that for a
fixed optimal target $\sfq^*$, the optimal policy satisfies:
\[
\pi^*(y) \propto \pi_0(y) \sfq^*(\EX(y)).
\]
Equating these two forms immediately yields the relation
$\sfq^*(\EX(y)) \propto \exp(r^*(y)/\beta)$. This proves that learning
the optimal target distribution $\sfq^*$ is equivalent to learning the
implicit reward function of the underlying decision process.
\end{proof}


\begin{figure*}[!t]
  \centering
  \includegraphics[scale = .155]{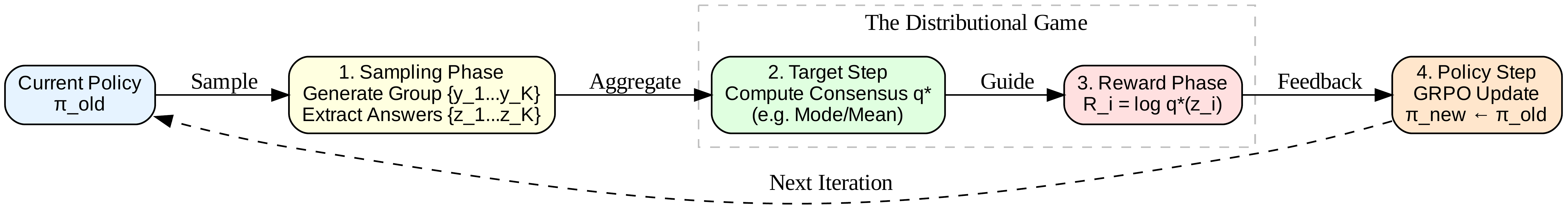}
  \caption{The Game-Theoretic Alignment Loop. This diagram illustrates
    the alternating best-response dynamics used in our GRPO-based
    algorithms. (1) The \emph{Policy} generates a group of traces. (2)
    The \emph{Target Step} solves the game by aggregating these
    outputs into an optimal target distribution $\sfq^*$ (e.g., via
    consensus for coherence or inverse-frequency for diversity). (3)
    This target defines the reward signal. (4) The \emph{Policy Step}
    updates the model to align with the target.}
    \label{fig:game_loop}
\end{figure*}

\subsection{Generic Algorithm: \algname{Game-GRPO}}

Based on this justification, we define a generic algorithm,
\algname{Game-GRPO}. At each step, we first compute the optimal Target
$\sfq^*$ (Target Step), which defines the reward $R$. We then compute
advantages and update the policy (Policy Step).

The Advantage $A_i$ for a trace $y_i$ is defined as its standardized
reward relative to the group statistics:
\begin{equation}
\label{eq:advantage-game}
A_i = \frac{R(y_i) - \text{Mean}(\{R(y_j)\}_{j=1}^K)}{\text{StdDev}(\{R(y_j)\}_{j=1}^K)
  + \e}.
\end{equation}
The policy update maximizes the surrogate objective:
{\small
\[
  \cL_{\text{GRPO}}(\pi)
  = \E_{x} \bracket*{\E_{G \sim \pi_{\text{old}}} \bracket*{
      \frac{1}{K} \sum_{i=1}^K \frac{\pi(y_i|x)}{\pi_{\text{old}}(y_i|x)} A_i
      - \beta \KL(\pi \| \pi_{\text{ref}}) }}.
\]
}

Thus, our general algorithmic framework consists of alternating two steps (illustrated in Figure~\ref{fig:game_loop}):

1.  \textbf{Target Step:} Estimate the optimal $\sfq^*$ from the
sampled group $G$ (e.g., via majority vote or centrality).

2.  \textbf{Policy Step:} Calculate rewards
$R_i = \log \sfq^*(\EX(y_i))$ and perform a GRPO update.


%% file: diversity.tex
\section{Diversity-Promotion: Problem and Algorithms}
\label{sec:diversity}

In domains like creative generation or exploratory search, the goal is
to prevent mode collapse and encourage the policy to cover a broad
distribution of valid answers, that is \emph{diversity}.

\subsection{The Inverse-Frequency Heuristic}

A prominent recent approach to this problem is the
\emph{Inverse-Frequency Reward} proposed by \citet{Li2025}. To
combat mode collapse, they introduce a heuristic reward that penalizes
frequent answers within a sampled group.

Formally, given a group of traces $G = \{y_1, \dots, y_K\}$ generating
answers $z_i = \EX(y_i)$, they define the reward for a trace $y_i$ as
the negative log-frequency of its answer:
\begin{equation}
\label{eq:inverse_freq}
R_{\text{inv-freq}}(y_i)
= -\log \paren*{ \frac{\text{Count}(z_i)}{K} }.
\end{equation}
Intuitively, this mechanism assigns high rewards to rare answers
(high surprisal) and low rewards to common ones, pushing the policy to
spread its mass to the tails of the distribution. While empirically
effective, this method was originally motivated primarily by intuition
regarding exploration.

\subsection{Theoretical Grounding via Distributional Alignment}

Our framework provides a rigorous theoretical derivation for this
algorithm. Within the ALFT framework, diversity is not a heuristic but
a specific instantiation of the objective function. We will refer
to the resulting algorithm by \algname{Diversity-GRPO}.

\subsubsection{Recovering Heuristic GRPO-Based Algorithm}

Let the answer-level regularization functional be the entropy
functional, $\cR(\nu) = \lambda H(\nu)$. The gradient of the entropy
with respect to the marginal distribution is:
\[
  \nabla H(\nu_\pi)
  = -\E_{z \sim \nu_\pi} \bracket*{ \nabla \log \nu_\pi(z) (\log \nu_\pi(z) + 1) }.
\]
We now derive the gradient of the entropy functional
$\cR(\nu_\pi) = H(\nu_\pi)$ with respect to the policy parameters. Let
$\nu_\pi(z)$ denote the marginal probability of answer $z$, then we
can write:
\begin{align*}
  \nabla H(\nu_\pi)
  & = \nabla \paren*{ - \sum_{z \in \mathcal{Z}} \nu_\pi(z) \log \nu_\pi(z) } \\
  & = - \sum_{z \in \mathcal{Z}} \nabla \paren*{ \nu_\pi(z) \log \nu_\pi(z) } \\
  & = - \sum_{z \in \mathcal{Z}} \nabla \nu_\pi(z) (\log \nu_\pi(z) + 1) \\
  & = - \sum_{z \in \mathcal{Z}} \nu_\pi(z) \nabla \log \nu_\pi(z) (\log \nu_\pi(z) + 1)
    \tag{$\log$-derivative identity:  $\nabla \nu = \nu \nabla \log \nu$} \\
  & = -\E_{z \sim \nu_\pi} \bracket*{ \nabla \log \nu_\pi(z) (\log \nu_\pi(z) + 1) }.
\end{align*}
This derivation shows the implicit reward structure of the maximum
entropy objective. Comparing this expression to the standard policy
gradient form $\nabla J = \E [\nabla \log \pi \cdot R]$, we identify
the effective reward signal $R(z)$ as the term scaling the gradient:
\[
R(z) = -(\log \nu_\pi(z) + 1) \approx -\log \nu_\pi(z).
\]
Since $\log \nu_\pi(z)$ is negative, the term $-\log \nu_\pi(z)$ is
positive. This means that answers with low probability (rare
answers) generate a large positive reward, while answers with high
probability (common answers) generate a smaller reward. Optimizing
this objective therefore incentivizes the policy to shift probability
mass from frequent answers to rare ones, theoretically justifying the
\emph{Inverse-Frequency} heuristic used for diversity promotion.

This implies that to maximize entropy via stochastic gradient descent,
the reward signal for a trace $y$ leading to answer $z$ must be
proportional to its \emph{surprisal} (negative log-likelihood):
\[
R_{\text{theory}}(y) \propto -\log \nu_\pi(\EX(y)).
\]
To implement this in the GRPO setting, we must estimate the current
marginal $\nu_\pi$ from the sampled group $G$. Using the Arithmetic
Mean (empirical frequency) estimator discussed in Section \ref{app:consensus_target}, we have
$\h \nu_\pi(z) \approx \frac{\text{Count}(z)}{K}$. Substituting this
into the theoretical reward yields:
\[
  R_{\text{theory}}(y_i) \approx -\log \paren*{
    \frac{\text{Count}(\EX(y_i))}{K} }.
\]
This demonstrates that the Inverse-Frequency Reward of
\citet{Li2025} is exactly a valid GRPO update step for the
\emph{Max-Entropy Distributional Alignment Game}.

\subsubsection{Alternative \algname{Diversity-GRPO} algorithms}

The game-theoretic perspective reveals that the inverse-frequency
reward is just one of many possible diversity-promoting objectives. By
changing the functional $\cR$, we can derive alternative algorithms
instances of our general \algname{Diversity-GRPO} algorithm, that may offer
better stability or properties:

\begin{enumerate}

\item \emph{Gini Diversity}: If we instead maximize the Gini index
  $\cR(\nu) = 1 - \sum_z \nu(z)^2$ (collision probability),
  the gradient yields a linear minority reward:
  \[
    R_{\text{gini}}(y_i) = 1 - \frac{\text{Count}(\EX(y_i))}{K}.
  \]
  Unlike the log-reward, this is bounded and does not explode for
  unique answers (where count is 1), potentially offering more stable
  optimization for small group sizes.

\item \emph{Targeted Exploration}: We can replace the uniform entropy
  objective with a divergence relative to a specific prior
  $Q_{\textrm{prior}}$. The game then drives the policy to match
  $Q_{\textrm{prior}}$ rather than just spreading mass uniformly. The
  reward becomes
  $R(y_i) = \log Q_{\textrm{prior}}(\EX(y_i)) - \log \h
  \nu_\pi(\EX(y_i))$, acting as an importance-weighted exploration
  signal.
\end{enumerate}

\begin{algorithm}[H]
\caption{\algname{Diversity-GRPO} (General Formulation)}
\label{alg:diversity_grpo}
\begin{algorithmic}[1]
  \REQUIRE Dataset $\cD$, Policy $\pi$, Group size $K$,
  \textbf{Objective} (Entropy or Gini).
\FOR{each training step}
    \STATE Sample batch $B \sim \cD$.
    \FOR{each $x \in B$}
    \STATE \textbf{1. Group Sampling:} Sample $K$ traces $\{y_1, \dots, y_K\}
    \sim \pi(\cdot|x)$.
        \STATE \textbf{2. Target Step (Marginal Estimation):}
        \STATE Estimate empirical marginals:
        $\h \nu(z) = \frac{\text{Count}(z)}{K}$.
        \STATE \textbf{3. Reward Calculation:}
        \IF{\textbf{Objective} is Entropy \citep{Li2025}}
        \STATE $R_k = -\log \h \nu(\EX(y_k))$
        \quad \textit{// Inverse-Frequency (High Variance)}
        \ELSIF{\textbf{Objective} is Gini}
        \STATE $R_k = 1 - \h \nu(\EX(y_k))$
        \quad \textit{// Probability of Collision (Bounded)}
        \ENDIF
        \STATE Compute advantages $A_k$ via Eq. \eqref{eq:advantage}.
    \ENDFOR
    \STATE \textbf{4. Update:} Optimize $\pi$ using GRPO.
\ENDFOR
\end{algorithmic}
\end{algorithm}

Algorithm~\ref{alg:diversity_grpo} illustrates the generality of our
framework. By selecting the Entropy objective, we recover the
inverse-frequency method of \citet{Li2025}. However, our
game-theoretic derivation reveals that we can easily substitute the
Gini objective to obtain a bounded, lower-variance reward signal
($1 - \h \nu$) without changing the underlying algorithm.

%% file: self-improvement.tex
\section{Answer-Level Self-Improvement via Coherence}
\label{sec:coherence}

\subsection{Background on Self-Improvement via Coherence}
\label{sec:background_coherence}

Here, we briefly give some background on self-improvement via
coherence. A more extensive description is given in
Appendix~\ref{app:trace_level_coherence}.

Standard self-improvement methods typically rely on heuristic
filtering or scalar rewards. In contrast, the coherence framework proposed by \citet{MohriSchneiderWu2025} offers a
rigorous geometric approach grounded in the principle of
\emph{coherence}. The core idea is that a reliable model should
produce mutually consistent distributions for inputs that are
semantically equivalent (e.g., a prompt $x$ and its paraphrase
$\Phi(x)$). Formally, let $\Phi: \sX \rightarrow \sX$ be a task-preserving transformation\footnote{Following \cite{MohriSchneiderWu2025}, we assume for mathematical simplicity that $\Phi$ is an involution, i.e., that $\Phi(\Phi(x)) = x$.} The set of \emph{coherent policies} is defined as
$\mathcal{C}_{\text{coh}} = \{ \pi \mid \nu_{\pi}(\cdot|x) =
\nu_{\pi}(\cdot|\Phi(x)) \}$. 

Importantly, as with the other answer-level tasks in this paper, it is the induced distribution $\nu_{\pi}$ on answers that we require to be coherent, as opposed to the immediate distribution $\pi(\cdot | x)$ on sequences $y$. In particular, for a task like a math problem, $y$ represents the chain-of-thought used to solve the problem and $z$ the numerical result. Our training algorithm should minimize the incoherence of the answer distributions on two equivalent problems $x$ and $\phi(x)$ while maintaining the diversity of valid reasoning paths in $\sY$.

Self-improvement is cast as a constrained optimization problem:
finding the coherent policy $\h \pi$ that is closest to the initial
baseline $\pi_0$ measured by a Bregman divergence $D_F$ (such as the
KL divergence):
\[
  \h \pi
  = \argmin_{\pi \in \mathcal{C}_{\text{coh}}}
  \E_x [ D_F(\pi(\cdot|x) \parallel \pi_0(\cdot|x)) ].
\]
This operation is a \emph{Bregman Projection}. \citet{MohriSchneiderWu2025}
prove that this projection guarantees \emph{monotonic improvement}:
the new policy $\h \pi$ is strictly closer to the optimal ground-truth
policy $\pi^*$ than the baseline $\pi_0$ was, provided $\pi^*$ is
itself coherent. 

This minimization problem can also be captured as a slight generalization of the ALFT framework, where we allow the functional $\cR$ (previously a functional acting on distributions over $\cZ$ ($\nu_{\pi}(\cdot | x)$)) to act on the collection of distributions in an equivalence class. In particular, if we define $\cR(\{\nu_\pi(\cdot|x')\}_{x' \in \cO_x})$ to equal $0$ if all distributions $\nu_{\pi}(\cdot|x')$ are equal and $+\infty$ otherwise (i.e., $\cR$ is the indicator function of the restriction of $\cC_{\text{coh}}$ to the orbit of $x$), then the above minimization problem agrees with the minimax problem defined in Theorem~\ref{thm:consistency}. The corresponding consensus target $\sfq^*$ is given by the geometric mean $\sfq^*(z)
= \frac{1}{Z} \paren*{ \prod_{x' \in \cO_x} \nu(z | x') }^{\frac{1}{|\cO_{x}|}}$, where $Z$ is a normalizing factor.

\subsection{Theoretical Analysis of Consensus Targets}
\label{sec:consensus_theory}

The core of the Coherence Game lies in defining the optimal Target
$\sfq^*$ that the policy must match. While the \emph{Geometric Mean}
is the exact solution for minimizing KL divergence, it is
computationally intractable for large output spaces due to the
partition function $Z$ \citep{MohriSchneiderWu2025}.  To derive a
practical algorithm, we analyze two tractable approximations:

\paragraph{1. The Arithmetic Mean.}  We can relax the objective to
minimize the squared Euclidean distance, yielding the Arithmetic Mean
$\sfq_{\text{AM}}(z) = \frac{1}{|\cO_{x}|} \sum _{x' \in \cO_{x}} \nu(z | x')$. Crucially, we can
prove that this relaxation remains theoretically sound.

Let $K = |\cO_{x}|$ and let $\{\nu_1, \nu_2, \dots, \nu_K\}$ be an enumeration of the distributions $\nu(z | x')$ for $x' \in \cO_{x}$. Let
$H^2(\{\nu_k(z) \}) = 1 - \sum_{z \in \sZ} \paren*{ \prod_{k =
    1}^K \nu_k(z) }^{1/K}$ denote the Generalized Squared
Hellinger (GSH) distance, which is a measure of the diversity of the
distributions $\nu_k$.
Let the improvement of a solution $\pi$ with respect to the baseline
$\pi_0$ be defined as
{\small
\[
 \Improv(\pi) = \E\bracket*{\KL(\pi^*(x) \parallel \pi_0(x))} -
 \E\bracket*{\KL(\pi^*(x) \parallel \pi(x))}.
\]
}
Then, the following guarantee holds.

\begin{theorem}[Stability of Arithmetic Consensus]
  Let $\h \pi_{\text{Hybrid}}$ be the policy resulting from projection
  onto the Arithmetic Mean. Its improvement degradation relative to
  the ideal Geometric Mean projection is bounded by the empirical
  diversity of the sample group:
{\small
\[
  \Improv(\h \pi_{\text{Hybrid}})
  \geq \Improv(\h \pi_{\text{Ideal}})
  - 2 L B \, \E_{x} \bracket*{ H^2(\{\nu_k(\cdot | x)\}) },
\]
}
where $L$ is a Lipschitz constant and $B$ upper bounds the
gradient norms of $\log \pi(y)$.
\end{theorem}
This result (proven in Appendix~\ref{app:arithmetic_approximation})
guarantees that shifting from the geometric (ideal) to the arithmetic
mean incurs only a small loss when the sample does not admit too large
a diversity.

\paragraph{2. Majority Vote.}  While the Arithmetic Mean is
stable, averaging distributions inherently increases entropy \citep{CoverThomas2006}. For reasoning tasks where a
single correct answer exists, we desire a sharpened target to
drive the policy toward decisive self-consistency.  Therefore, our
practical algorithm adopts the \emph{Majority Vote} (Mode), which
can be viewed as the arithmetic mean followed by a hard $\argmax$
operator:
\[
  \sfq^*(z)
  = \mathds{1}(z = \argmax \sfq_{\text{AM}}(z)).
\]
This choice sacrifices the granular probability information of the
Arithmetic Mean in exchange for a lower-entropy target that drives the
policy toward decisive self-consistency.

In Appendix~\ref{app:consensus_target}, we present a more extensive
discussion of the properties of different Target choices.

\subsection{\algname{Coherence-GRPO} Algorithm}

Our generic GRPO-based algorithm, \algname{Game-GRPO}, can be
instantiated in the case of the \emph{Coherence Game} to provide an
algorithm tackling tasks with a discrete answer space $\sZ$, such as
mathematical reasoning.

\textbf{Target Step (Orbit Consensus).} To enforce coherence, the
Target $\sfq^*$ must represent the consensus not just of the model's
current output for $x$, but of the entire equivalence class
$\text{Orbit}(x)$. Computationally, we approximate this by sampling a
small set of inputs from the orbit (e.g., $x$ and a paraphrase
$x' = \Phi(x)$). We generate groups of traces for both, pool their
answers, and compute the Global Mode of this pooled set.
The reward for a trace $y$ given $x$ is then determined by whether it
matches this \emph{global} consensus $z^*$, rather than merely the
local consensus of $x$. This penalizes the model if $\pi(\cdot|x)$ and
$\pi(\cdot|x')$ disagree, directly optimizing the coherence objective.

\textbf{Reward Definition.}  Based on the analysis in
Section~\ref{sec:consensus_theory}, we adopt the \textbf{Majority
  Vote} (Mode) target to enforce sharpening. This sets $\sfq^*$ to be
a Dirac distribution centered on the most frequent answer $z^*$. The
reward is derived directly from this sharpened target (setting $R_i = 1$ if $\EX(y_i) = \text{Mode}(\{\EX(y_j)\})$ and $0$ otherwise).

The reward for a trace $y$ given $x$ is thus determined by whether it
matches this \emph{global} consensus $z^*$, rather than merely the
local consensus of $x$. This penalizes the model if $\pi(\cdot|x)$ and
$\pi(\cdot|x')$ disagree, directly optimizing the coherence objective.

\begin{algorithm}[htb]
\caption{\algname{Coherence-GRPO} (Orbit-Level Consensus)}
\label{alg:coherence_grpo}
\begin{algorithmic}[1]
  \REQUIRE Dataset $\cD$, Policy $\pi$, Extractor $\EX$, Transformation $\Phi$
  (e.g., paraphrase).
  \FOR{each training step}
    \STATE Sample batch of seeds $B \sim \cD$.
    \FOR{each seed $x \in B$}
        \STATE \textbf{1. Orbit Sampling:}
        \STATE Generate orbit inputs $\mathcal{O}_x = \{x, \Phi(x)\}$.
        \quad \textit{// e.g., Original + Paraphrase}
        \STATE Sample $K$ traces for \emph{each} input in $\mathcal{O}_x$:
        \STATE $\quad G = \{ (y_{i,j}, x^{(j)}) \mid x^{(j)} \in \mathcal{O}_x, i \in \{1..K\}, y_{i,j} \sim \pi(\cdot|x^{(j)}) \}$.
        \STATE \textbf{2. Target Step (Orbit Consensus):}
        \STATE Extract all answers: $\mathcal{Z}_G = \{ \EX(y) \mid (y, \cdot) \in G \}$.
        \STATE Compute \emph{Global Consensus} across the orbit:
        \STATE $\quad z^* = \text{Mode}(\mathcal{Z}_G)$.
        \STATE \textbf{3. Reward Calculation:}
        \STATE For each trace $y_{i,j}$ from input $x^{(j)}$:
        \STATE $\quad R_{i,j} = \mathds{1}(\EX(y_{i,j}) = z^*)$.
        \quad \textit{// Reward match to GLOBAL consensus}
        \STATE Compute advantages $A_{i,j}$ (standardized within the orbit group $G$).
    \ENDFOR
    \STATE \textbf{4. Update:} Optimize $\pi$ using GRPO on all generated traces.
  \ENDFOR
\end{algorithmic}
\end{algorithm}

\subsection{Pairwise Answer-Level Coherence}

The previous discussion supposes the existence of a deterministic
extraction function $\EX(y) \to z$ that maps a reasoning trace to a
canonical discrete answer. In Appendix~\ref{sec:pairwise}, we develop an algorithm (\algname{Pairwise-GRPO}) that only requires access to a pairwise disagreement function $d(y, y')$ quantifying the semantic distance between two traces.

%% file: pairwise.tex
\section{Pairwise Answer-Level Coherence}\label{sec:pairwise}

In previous sections, we assumed the existence of a deterministic
extraction function $\EX(y) \to z$ that maps a reasoning trace to a
canonical discrete answer. In many complex domains, such as open-ended
question answering, summarization, or semantic code retrieval, such a
function is difficult to define or brittle to implement. A parser may
fail to extract an answer even when the reasoning is correct, or two
string-distinct answers might be semantically equivalent. Even for
some math problems, it may difficult to come up with an accurate
extraction function.

However, in these domains, it is often feasible to define a
\emph{pairwise disagreement function} $\sfd(y, y') \in [0, 1]$ that
quantifies the semantic distance between two traces. $\sfd(y, y') = 0$
implies perfect consistency, while $\sfd(y, y') = 1$ implies total
contradiction.

Note that the answer-level formulation via an extraction function
$\EX$ can always be cast as a pairwise one using $\sfd(y, y')
= \Ind\paren*{\EX(y) \neq \EX(y')}$.

We regard the pairwise formulation as a critical extension of the
standard framework, as it adapts the Coherence Game to open-ended
domains where exact answer matching is infeasible. Moreover, this
pairwise perspective naturally generalizes to other answer-level
games.

\subsection{Problem Formulation}

The goal of self-improvement is to ensure that the model's outputs are
semantically consistent across task-preserving transformations. We
therefore define the pairwise incoherence penalty over the orbit of an
input $x$:
\begin{equation}
  \Inc(\pi_\theta)
  = \E_{x' \sim \text{orbit}(x)} \bracket*{\E_{y \sim \pi(\cdot|x), y' \sim \pi(\cdot|x')}
    \bracket*{ d(y, y') }}.
\end{equation}
This forces the distribution $\pi(\cdot|x)$ to be semantically aligned
with $\pi(\cdot|x')$, rather than just collapsing to a low-entropy
state for a single prompt.

\subsection{\algname{Pairwise-GRPO} Algorithm}

Target Step (Best Response). Since we cannot compute a mode,
we view the Target $\sfq^*$ as the ``center of mass'' of the
semantic distribution. We define a pairwise distance metric
$d(y, y')$. The optimal Target assigns high probability to traces that
minimize the average transport cost to all other traces in the group.

Reward Definition. The reward is the \emph{centrality} of the trace, which
proxies for $\log \sfq^*$:
\[
R_i = \frac{1}{K} \sum_{j=1}^K (1 - d(y_i, y_j)).
\]
Thus, high $R_i$ implies $y_i$ is close to the group consensus $\sfq^*$.

\begin{proposition}[Derivation of Orbit-Centrality]
\label{prop:orbit_energy}
Let the orbit-level incoherence objective be defined as the expected
pairwise distance between traces generated from inputs in the same
equivalence class $\mathcal{O}_x$:
\[
  \mathcal{J}_{\text{orbit}}(\pi)
  = \E_{x^{(a)}, x^{(b)} \sim \mathcal{O}_x} \bracket*{
  \E_{y_a \sim \pi(\cdot|x^{(a)}), y_b \sim \pi(\cdot|x^{(b)})} \bracket*{d(y_a, y_b)} }.
\]
The gradient of this objective with respect to the policy parameters $\theta$ satisfies:
\begin{equation}
  \nabla_\theta \mathcal{J}_{\text{orbit}}(\pi)
  = 2 \E_{x \sim \mathcal{O}_x} \bracket[\Bigg]{\E_{y \sim \pi(\cdot|x)} \bracket[\Bigg]{
    \nabla_\theta \log \pi(y|x) \, \underbrace{\E_{x' \sim \mathcal{O}_x} \bracket*{
      \E_{y' \sim \pi(\cdot|x')} [d(y, y')]}}_{\text{Orbit Incoherence}} }}.
\end{equation}
To minimize this objective via policy gradient, the optimal reward
signal $R^*(y)$ must be proportional to the negative of the gradient
coefficient. Thus, we define the reward as the negative expected
distance to the orbit:
\[
  R^*(y)
  \propto - \E_{x' \sim \mathcal{O}_x} \bracket*{\E_{y' \sim \pi(\cdot|x')} [d(y, y')]}.
\]
\end{proposition}

\begin{proof}
  Let the objective for a fixed pair of inputs $(x^{(a)}, x^{(b)})$ be
  $J_{a,b}(\pi) = \E_{y_a \sim \pi_a, y_b \sim \pi_b} [d(y_a, y_b)]$,
  where $\pi_k = \pi(\cdot|x^{(k)})$. We compute the gradient using
  the log-derivative trick:
\begin{align*}
\nabla_\theta J_{a,b}(\pi) 
&= \nabla_\theta \sum_{y_a, y_b} \pi_a(y_a) \pi_b(y_b) d(y_a, y_b) \\
&= \sum_{y_a, y_b} \paren*{ \pi_a(y_a) [\nabla \log \pi_a(y_a)] \pi_b(y_b)
  + \pi_a(y_a) \pi_b(y_b) [\nabla \log \pi_b(y_b)] } d(y_a, y_b) \\
&= \E_{y_a, y_b} \bracket*{ \nabla \log \pi_a(y_a) \cdot d(y_a, y_b) }
  + \E_{y_a, y_b} \bracket*{ \nabla \log \pi_b(y_b) \cdot d(y_a, y_b) }.
\end{align*}
The total gradient is the sum over all pairs $a, b$. By symmetry, the
total contribution for a specific trace $y$ generated from input $x$
appears twice in the summation (once as $y_a$, once as $y_b$). Thus,
the gradient coefficient for $\nabla \log \pi(y|x)$ is exactly:
\[
2 \times \sum_{x' \in \mathcal{O}_x} \E_{y' \sim \pi(\cdot|x')} [d(y, y')].
\]
We define the reward $R(y)$ to be proportional to the negative of this
term. The factor of 2 is a constant scaler that is absorbed by the
learning rate or advantage standardization in GRPO.
\end{proof}

\begin{algorithm}[H]
\caption{\algname{Pairwise-GRPO} (Orbit-Level Consistency)}
\label{alg:pairwise_grpo}
\begin{algorithmic}[1]
\REQUIRE Dataset $\cD$, Group size $K$, Policy $\pi$, Metric $d$, Transformation $\Phi$.
\FOR{each training step}
    \STATE Sample batch of seeds $B \sim \cD$.
    \FOR{each seed $x \in B$}
        \STATE \textbf{1. Orbit Sampling:}
        \STATE Generate orbit inputs $\mathcal{O}_x = \{x, \Phi(x)\}$.
        \STATE Sample $K$ traces for each input:
        \STATE $\quad G = \{ (y_{i,j}, x^{(j)}) \mid x^{(j)} \in \mathcal{O}_x, i \in \{1..K\}, y_{i,j} \sim \pi(\cdot|x^{(j)}) \}$.
        \STATE \textbf{2. Target Step (Orbit Centrality):}
        \STATE \textit{// Reward is centrality relative to the ENTIRE orbit group}
        \FOR{each trace $y \in G$}
            \STATE $R(y) = \frac{1}{|G|} \sum_{y' \in G} (1 - d(y, y'))$.
        \ENDFOR
        \STATE \textbf{3. Advantage Calculation:}
        \STATE Compute standardized advantages $A_k$ across the orbit group $G$.
    \ENDFOR
    \STATE \textbf{4. Update:} Optimize $\pi$ using GRPO with advantages $A_k$.
\ENDFOR
\end{algorithmic}
\end{algorithm}

%% file: safety.tex
\section{Distributional Constraint Satisfaction: Safety and Fairness}\label{sec:safety}

While the Diversity and Coherence games optimize for specific shapes
of the distribution (spread vs. point-mass), a third critical class of
problems involves \emph{Distributional Constraints}. In real-world
deployment, we often require the aggregate behavior of the model to
satisfy safety or fairness guarantees, even if individual answers are
technically correct.

Here are some examples for both types: (1) Safety/Toxicity: We may
require that the expected toxicity score of generated answers remains
below some threshold $\e$, or that the rate of refusal for benign
prompts is minimized; (2) Fairness: A medical diagnosis model may be
required to yield positive predictions at equal rates across
demographic groups (Demographic Parity). This is a constraint on the
marginal distribution of answers $\nu_\pi$, not on any single trace.

\subsection{Problem Formulation}

We formulate this as a constrained optimization problem. We wish to
minimize the divergence from the reference policy $\pi_0$ subject to
the marginal distribution $\nu_\pi$ remaining within a valid set
$\cC$:
\begin{equation}
\label{eq:constrained_primal}
\min_{\pi \in \Pi} \quad \KL(\pi \| \pi_0) \quad \text{s.t.} \quad \nu_\pi \in \cC.
\end{equation}
Typically, $\cC$ is defined by linear constraints on feature
functions $\bc(z)$ (e.g., toxicity classifiers or demographic
indicators):
\[
\cC = \left\{ \nu \in \Delta(\sZ) \mid \E_{z \sim \nu}[\bc(z)] \le \bb \right\}.
\]
In our ALFT framework, this maps to the general problem
\eqref{eq:primal} by defining the answer-level functional
$\mathcal{R}(\nu)$ as the convex indicator function of the set
$\cC$ (taking value $0$ if $\nu \in \cC$ and $+\infty$
otherwise).

\subsection{Game Interpretation: The Projection Game}

In the game-theoretic dual, the role of the Target $\sfq$ is to
act as the \emph{Corrected Distribution}.  Since $\cR$ is the
indicator of $\cC$, the optimal Target $\sfq^*$ in the game becomes
the Information Projection (I-Projection) of the current policy's
marginal $\nu_\pi$ onto the constraint set $\cC$.
\begin{equation}
\label{eq:safety_target}
\sfq^*
= \argmin_{\sfq \in \cC} \KL(\sfq \parallel \nu_\pi).
\end{equation}
Intuitively, the Target looks at the current noisy or unsafe behavior
of the policy ($\nu_\pi$) and finds the closest distribution that
satisfies the safety/fairness constraints. The Policy then updates to
mimic this safe Target.

\subsection{\algname{Safety-GRPO} algorithm}

Computing the exact projection $\sfq^*$ is difficult. However,
for linear constraints $\E[\bc] \le \bb$, the projection has a known
exponential form parametrized by Lagrange multipliers
$\lambda \in \mathbb{R}^d_{\ge 0}$:
\[
\sfq^*(z) \propto \nu_\pi(z) \exp\paren*{ -\lambda^\top \bc(z) }.
\]
Substituting this target into the GRPO reward definition
($R(y) = \log \sfq^*(\EX(y))$), we derive a computationally
efficient algorithm based on Dual Ascent.

\textbf{Reward Structure.}
The effective reward for a trace $y$ is simply the negative weighted
constraint violation:
\begin{equation}
\label{eq:safety_reward}
R_{\text{safe}}(y) = -\lambda^\top \bc(\EX(y)).
\end{equation}
Here, $\lambda$ acts as a dynamic penalty weight. The algorithm alternates between two steps:
\begin{enumerate}

\item Policy Step (GRPO): Update $\pi$ using rewards
  $R_{\text{safe}}(y)$ to minimize constraint violations.

\item Dual Step (Update $\lambda$): Update the penalty weights
  $\lambda$ via gradient ascent to enforce the constraints strictly:
  \[
    \lambda \leftarrow \max\paren*{0, \lambda
      + \alpha \paren*{ \E_{z \sim \pi}[\bc(z)] - \bb } }.
  \]
\end{enumerate}
This derivation further shows that standard approaches to constrained
RL (like Lagrangian relaxation) are, in fact, finding the Nash
Equilibrium of the Distributional Alignment Game where the Target
plays the strategy of the \emph{closest safe distribution}.

This leads to a theoretically grounded algorithm we refer to
as \algname{Safety-GRPO}.

\begin{algorithm}[H]
\caption{\algname{Safety-GRPO} (Primal-Dual)}
\label{alg:safety_grpo}
\begin{algorithmic}[1]
  \REQUIRE Dataset $\cD$, Policy $\pi$, Constraints
  $\bc(z) \le \bb$, Step size $\alpha$.
\STATE Initialize Lagrange multipliers $\lambda \leftarrow \mathbf{0}$.
\FOR{each training step}
    \STATE Sample batch $B \sim \cD$.
    \FOR{each $x \in B$}
        \STATE Sample $K$ traces $\{y_1, \dots, y_K\} \sim \pi(\cdot|x)$.
        \STATE Extract answers $z_k = \EX(y_k)$ and features $\bc(z_k)$.
        \STATE \textbf{Target Step:} Implicitly defined by penalty weights $\lambda$.
        \STATE \textbf{Reward Calculation:}
        \STATE \quad $R_k = -\lambda^\top \bc(z_k)$ \quad \textit{// Penalize constraint violations}
        \STATE Compute advantages $A_k$ via Eq. \eqref{eq:advantage}.
    \ENDFOR
    \STATE \textbf{Policy Step:} Update $\pi$ using GRPO with advantages $A_k$.
    \STATE \textbf{Dual Step:} Update penalties via gradient ascent:
    \STATE \quad $\h \bc = \frac{1}{|B| K} \sum_{x \in B} \sum_{k=1}^K \bc(z_{k})$ \quad \textit{// Estimate expected violation}
    \STATE \quad $\lambda \gets \max(0, \lambda + \alpha (\h \bc - \bb))$.
\ENDFOR
\end{algorithmic}
\end{algorithm}

%% file: experiments.tex
\section{Experimental results}
\label{sec:experiments}

\begin{table*}[t]
\centering
\small
\begin{tabular}{lccccc}
\hline
Algorithm & Model & Baseline ACC (95\% CI) & Our ACC (95\% CI) & Abs.\ $\uparrow$ (pp) & Rel.\ $\uparrow$ \\
\hline
Pairwise-GRPO & Qwen-3B   & 75.06 [72.65, 77.32] & 79.61 [77.35, 81.69] & +4.55 & +6.06\% \\
 & Llama    & 66.19 [63.59, 68.69] & 72.71 [70.24, 75.04] & +6.52 & +9.85\% \\
 & Phi-3   & 73.69 [71.25, 76.00] & 82.87 [80.74, 84.80] & +9.18 & +12.46\% \\
\hline
Coherence-GRPO & Qwen-3B   & 75.06 [72.65, 77.32] & 80.36 [78.13, 82.42] & +5.30 & +7.06\% \\
 & Llama    & 66.19 [63.59, 68.69] & 69.37 [66.83, 71.80] & +3.18 & +4.80\% \\
 & Phi-3   & 73.69 [71.25, 76.00] & 81.50 [79.32, 83.50] & +7.81 & +10.60\% \\
\hline
\end{tabular}
\caption{GSM8K improvement (greedy decoding). Absolute improvement is in percentage points (pp); relative improvement is computed as $(\text{Our ACC} - \text{base})/\text{base}$.}
\label{tab:gsm8k_grpo_results}
\end{table*}

\begin{table*}[t]
\centering
\small
\begin{tabular}{lccccccc}
\hline
Algorithm & Model & Base EM & Our EM & Rel.\ $\uparrow$ (EM) & Base F1 (95\% CI) & Our F1 (95\% CI) & Rel.\ $\uparrow$ (F1) \\
\hline
Pairwise   & Qwen & 32.95 & 35.32 & +7.19\%  & 39.70 $\pm$ 0.67 & 40.28 $\pm$ 0.69 & +1.46\% \\
               & Llama  & 39.12 & 47.85 & +22.32\% & 48.56 $\pm$ 0.66 & 53.52 $\pm$ 0.70 & +10.22\% \\
               & Phi-3  & 32.03 & 45.50 & +42.06\% & 42.76 $\pm$ 0.63 & 50.51 $\pm$ 0.70 & +18.12\% \\
\hline
Coherence  & Qwen & 32.95 & 32.90 & -0.15\% & 39.70 $\pm$ 0.67 & 39.62 $\pm$ 0.67 & -0.20\% \\
               & Llama  & 39.12 & 40.25 & +2.89\%  & 48.56 $\pm$ 0.66 & 49.35 $\pm$ 0.66 & +1.63\% \\
               & Phi-3  & 32.03 & 32.00 & -0.09\% & 42.76 $\pm$ 0.63 & 42.79 $\pm$ 0.63 & +0.07\% \\
\hline
\end{tabular}
\caption{TriviaQA results. Relative improvement is computed as $(\text{ours} - \text{base})/\text{base}$. For F1, we report the 95\% CI half-width as $\pm$. For algorithms, Pairwise and Coherence stand for Pairwise-GRPO and Coherence-GRPO, respectively. }
\label{tab:triviaqa_grpo_results}
\end{table*}

We evaluate our framework in two stages. First, we validate the
variance reduction properties of the \algname{Game-GRPO} estimator in
a controlled synthetic environment with explicit ground-truth
heterogeneity. Second, we scale the approach to Answer-Level
Self-Improvement on large language models, demonstrating significant
gains on mathematical reasoning (GSM8K) and open-ended question
answering (TriviaQA).

\input{synthetic.tex}

\subsection{Answer-Level Coherence on LLMs (GSM8K \& TriviaQA)}
\label{sec:llm_exp}

We evaluate answer-level self-improvement via coherence on GSM8K and TriviaQA datasets. We use the official train split to construct paraphrase groups for training and report performance on the official test split with $1319$ questions. Each example is formatted as an instruction-following chat prompt with a fixed system message and a user message containing the math word problem.

We test \texttt{\small Qwen2.5-3B-Instruct}, \texttt{\small Phi-3-mini-4k\-instruct}, and \texttt{Llama-3.2-3B\-Instruct}, referred to as \texttt{\small Qwen}, \texttt{\small Phi3}, and \texttt{\small Llama}, respectively. We generate one paraphrase for each question in the training set with \texttt{\small Qwen} and greedy decoding. 

For all models, we fine-tune with QLoRA (4-bit NF4 quantization with double-quantization; compute dtype bfloat16 on GPU) and attach a LoRA adapter with $r=8$, $\alpha=16$, dropout 0.05 (target modules auto-detected among the usual projection layers).

For GSM8K, training uses Pairwise-GRPO on paired prompts (original + paraphrase) with 2 pairs per step and 32 rollouts per prompt. We set KL weight to 0.02, learning rate to $1e-5$, and max completion length to 512. We train for 3200 steps, disable batch dispatch/splitting in the accelerator. 
We report the results after 3200 steps in \Cref{tab:gsm8k_grpo_results}. Due to computation constraints, we use 1 pair per step and 24 rollouts per prompt for \texttt{\small Phi3}.


For TriviaQA, we trained with the same parameters. We additionally instantiate a embedder \texttt{\small SentenceTransformer}  on GPU when available, and use its embeddings to judge semantic similarity between generated answers. 
We test and report the improvement with GRPO in \Cref{tab:triviaqa_grpo_results}.

%% file: synthetic.tex
\subsection{Synthetic Validation: Variance Reduction in Many-to-One Mappings}
\label{sec:synthetic_exp}

\begin{figure}[t]
    \centering
    \includegraphics[width=\linewidth]{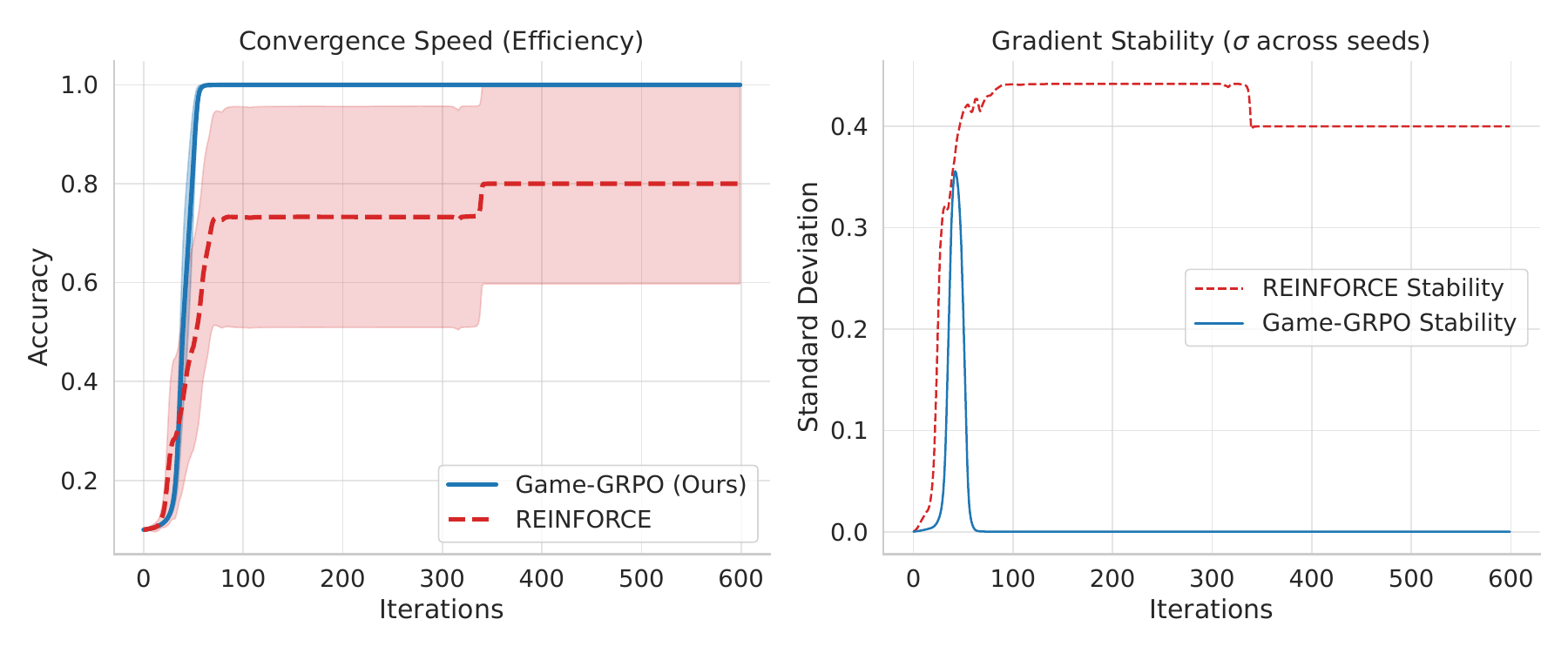}
    \caption{Robustness to Extreme Heterogeneity. We compare
      \algname{Game-GRPO} (Blue) against REINFORCE with a global
      baseline (Red) on a 1000-to-1 redundant mapping task. The
      environment switches between ``Easy'' (Bias 8.0) and ``Hard'' (Bias
      0.0) instances.  \emph{(Left)} \algname{Game-GRPO} isolates the
      true advantage signal, converging to $1.0$ accuracy. The
      baseline is overwhelmed by the $8\times$ bias, fluctuating
      around $0.8$.  \emph{(Right)} The stability plot reveals the
      mechanism: the global baseline induces massive gradient variance
      ($\sigma \approx 0.4$), whereas the group-relative baseline
      reduces variance to near zero, enabling deterministic-like
      convergence.}    \label{fig:synthetic_convergence}
\end{figure}

Before scaling to large language models, we validate our theoretical
claims in a controlled synthetic environment. A core challenge in ALFT
is the ``many-to-one'' nature of reasoning, where thousands of latent
traces $y$ map to a single answer $z$ \citep{Uesato2022}. Standard
gradient estimators struggle to assign credit in this setting,
especially when the task difficulty varies per instance (heterogeneous
difficulty).

\textbf{Experimental Setup.}
We construct a ``Redundant Mapping'' environment with a trace space
$|\sY|=10,000$ mapping to an answer space $|\sZ|=10$ (a $1000:1$
redundancy ratio). To simulate the extreme variance of real-world
prompts, we introduce a stochastic difficulty bias. At each step, the
environment is instantiated as either ``Easy'' (Base Reward $= 8.0$) or
``Hard'' (Base Reward $= 0.0$). A correct answer yields a marginal
signal of $+1.0$.  This creates a signal-to-Bias ratio of 1:8, meaning
the environmental noise is eight times larger than the learning
signal.

\textbf{Baselines.} We compare \algname{Game-GRPO} against a strong
REINFORCE baseline equipped with a global moving-average baseline.

\textbf{Results.} Figure~\ref{fig:synthetic_convergence} illustrates
the learning dynamics.
\begin{itemize}
\item \textbf{Robustness:} \algname{Game-GRPO} (Blue) effectively
  decouples the signal from the difficulty bias by computing the
  baseline dynamically from the group. It converges rapidly to optimal
  performance ($100\%$ accuracy) with negligible variance.

\item \textbf{Failure of Global Baselines:} The REINFORCE baseline
  (Red) fails to normalize the extreme heterogeneity. Because the
  global average ($\approx 4.0$) is too low for ``Easy'' tasks and too
  high for ``Hard'' tasks, the estimator suffers from catastrophic
  variance ($\sigma \approx 0.4$). It plateaus at sub-optimal
  performance ($\approx 80\%$), unable to fully resolve the credit
  assignment problem amidst the noise.

\end{itemize}

%% file: conclusion.tex
\section{Conclusion}

We introduced \emph{Distributional Alignment Games}, a variational
framework that resolves the fundamental computational bottleneck of
Answer-Level Fine-Tuning. By lifting the optimization problem from the
space of traces to a game between a Policy and a Target distribution,
we transform the intractable marginalization over reasoning paths into
a tractable projection problem.

This perspective provides more than just a computational trick; it
offers a unified theoretical lens for reasoning alignment. We have
demonstrated that distinct objectives—whether maximizing
\emph{diversity} for exploration or enforcing \emph{coherence} for
self-improvement—are merely different strategies played by the Target
distribution. This rigor allows us to ground recent heuristic
successes, such as inverse-frequency scaling and self-consistency, in
established convex duality theory.

Furthermore, we bridged the gap between theory and practice by
deriving \emph{Game-Theoretic GRPO}. We showed that the abstract
equilibrium conditions translate directly into scalable algorithms
like \emph{\algname{Coherence-GRPO}} and
\emph{\algname{Pairwise-GRPO}}, where the "reward" is simply the
log-likelihood of the optimal target. By extending this framework to
\emph{Pairwise Coherence} and \emph{Distributional Constraints}, we
have paved the way for aligning models in open-ended domains where
exact answers are undefined and safety guarantees are paramount.

%% file: appendix.tex
\input{omitted}

\section{Solving the Minimax Game via No-Regret or Best Response Dynamics}
\label{app:solving-minimax}

Crucially, the game formulation is computationally tractable.  It
offers a flexible theoretical basis that admits two distinct types of
solutions, depending on the computational constraints and the
properties of the functional $\cR$.

\paragraph{1. General Solution via No-Regret Dynamics.}  For some objectives, the game permits a convex-concave parametrization, allowing us employ techniques from no-regret learning to efficiently compute a Nash equilibriumm. Specifically, the objective  $\cG(\pi, \sfq)$ satisfies the following properties:
\begin{itemize}

\item \textbf{Strictly Convex in $\pi$}: The term $\KL(\pi \| \pi_0)$
  is strictly convex in $\pi$, while the coupling term
  $\E_{y \sim \pi}[\log \sfq(\EX(y))]$ is linear in $\pi$.

\item \textbf{Concave in Dual Variables}: The objective is concave
  with respect to the natural dual parameters $u$ (where
  $u = -\beta \log \sfq$).  While the composition with the non-linear
  mapping $u \to \sfq$ does not guarantee concavity in $\sfq$ for
  general functionals, the underlying convexity of $\cR^*$ ensures the
  problem is well-posed in the dual space.  For specific regularizers
  (e.g., entropy), concavity in $\sfq$ may also hold on the simplex.

\end{itemize}

This structure technically allows for solution via standard no-regret
dynamics (e.g., Mirror Descent or Follow-the-Regularized-Leader)
\citep{FreundSchapire1999,ShalevShwartz2012} on the dual variables
$u$.  By updating $\pi$ to minimize regret and $u$ to maximize regret,
the average iterates converge to the Nash Equilibrium at a rate of
$O(1/\sqrt{T})$. This provides a robust fallback for cases where exact
maximization is difficult.

\paragraph{2. Efficient Solution via Alternating Best Response.}
In the context of ALFT, we can adopt a faster strategy:
\emph{Alternating Best Response}.  Instead of incremental gradient
steps, we solve the inner optimization problems exactly. Crucially, we
perform the maximization directly over the target distribution $\sfq$
rather than the dual variable $u$. This is justified by the bijection
established in Section 2.2: finding the optimal
$\sfq^* \in \Delta(\sZ)$ is equivalent to finding the optimal dual
variable $u^*$ (modulo shift) that maximizes the concave dual
objective.

We thus iterate the following two steps:

\begin{enumerate}
\item \textbf{Target Step (Exact Dual Maximization):} Fix $\pi$ and
  find the optimal target $\sfq^*$.
  \[ \sfq^* = \argmax_{\sfq \in \Delta(\sZ)} \cG(\pi, \sfq). \]
  Because of the bijection between $\sfq$ and the equivalence classes
  of $u$, this step is equivalent to finding the exact maximizer of
  the dual objective $u^*$. For many functionals (e.g., entropy or
  Euclidean distance), this admits a closed-form solution (e.g.,
  centroid).

\item \textbf{Policy Step (Exact Primal Minimization):} Fix $\sfq^*$
  and update $\pi$ to match the trace-level target
  $\wt \sfq(y) \propto \pi_0(y)\sfq^*(\EX(y))$.
\[ \pi_{new} = \argmin_{\pi \in \Pi} \cG(\pi, \sfq^*). \]
This is a standard supervised learning task (KL projection).
\end{enumerate}

\textbf{Convergence.}  While alternating best response does not
converge for general games, it is guaranteed to converge for
\emph{strictly convex-concave} games where the subproblems are
solved exactly. Since our primal objective $\cJ(\pi)$ is strictly
convex (and the implied dual is strictly concave modulo shifts), this
alternating procedure converges to the unique global optimum
$(\pi^*, \sfq^*)$ \citep{Tseng2001}. This justifies our algorithmic
choice to iteratively estimate the target and project the policy.

\textbf{Connection to Variational EM.}
Structurally, our alternating best-response scheme resembles a
\emph{Variational Expectation-Maximization (EM)} procedure. The
\emph{Target Step} acts as an \emph{E-Step}, inferring the optimal
variational distribution $\sfq^*$, while the \emph{Policy Step} acts
as an \emph{M-Step}, updating $\pi$ to match these targets. However,
unlike standard EM which typically guarantees convergence only to a
stationary point, our formulation as a game derived from a convex primal
problem allows for stronger global convergence guarantees.

\section{Choice of the Consensus Target}
\label{app:consensus_target}

In the Coherence Game, the Target $\sfq^*$ serves as the ground truth
estimate that the policy attempts to match. Mathematically, this
target is the \emph{Bregman Centroid} of the distributions
$\{\nu_k\}_{k=1}^K$ implied by the sampled group. The specific choice
of divergence determines the nature of this consensus. We analyze
three candidates for $\sfq^*$: the Geometric Mean, the Arithmetic
Mean, and the Majority Vote.

\subsection{Geometric Mean (Ideal Consensus)}

A. Mathematical Definition.  For self-improvement, we seek a sharp
consensus that eliminates noise. This is formally captured by
minimizing the average KL Divergence:
\[
  \sfq^*
  = \argmin_{\sfq \in \Delta(\sZ)} \frac{1}{K} \sum_{k=1}^K \KL(\sfq \| \nu_k).
\]
The closed-form solution is the Geometric Mean:
\begin{equation}
\label{eq:geo_mean}
\sfq^*(z)
= \frac{1}{Z} \paren*{ \prod_{k=1}^K \nu_k(z) }^{\frac{1}{K}},
\quad \text{where }
Z
= \sum_{z' \in \sZ} \paren*{ \prod_{k=1}^K \nu_k(z') }^{\frac{1}{K}}.
\end{equation}

B. Approximation Status.  This is the exact solution for the strict
Coherence objective. It theoretically provides the strongest signal
for self-improvement because it enforces \emph{intersectional
  consistency}: if any trace assigns zero probability to an answer,
the consensus probability becomes zero.

C. Computational Properties.  Despite its theoretical appeal, the
Geometric Mean is \emph{computationally intractable} for two reasons:
(1) Normalization: Computing the partition function $Z$ requires
summing over the exponentially large space of all possible answers
$\sZ$; (2) Sample Sparsity: In a GRPO setting, we only observe
discrete samples $z_k$, not the full probability vectors $\nu_k$. If
we treat samples as one-hot distributions, the product term becomes
zero everywhere unless the group is unanimous, rendering the mean
degenerate.

D. Statistical Properties.  The Geometric Mean has the property of
\emph{Veto Power}. It acts as a strict filter, keeping only answers
supported by \emph{all} reasoning paths. While this removes
hallucinations effectively, it can be overly aggressive in early
training, leading to signal collapse if the model is not yet
consistent.

\subsection{The Arithmetic Mean (Robust Relaxation)}

A. Mathematical Definition.
The Arithmetic Mean corresponds to the standard mixture distribution:
\[
\ov \nu(z) = \frac{1}{K} \sum_{k=1}^K \nu_k(z).
\]
It is the unique minimizer for both the \emph{Forward KL Divergence}
($\sum \KL(\nu_k \| \sfq)$) and the Squared Euclidean Distance.

B. Approximation Status.  This is a relaxed solution. While it
minimizes a different divergence than the strict mode-seeking Reverse
KL, it preserves the support of the distribution.
\begin{itemize}
\item \textbf{Theoretical Guarantee:} As shown in
  Appendix~\ref{app:arithmetic_approximation}, iterative projection
  onto the Arithmetic Mean guarantees convergence to a coherent
  equilibrium, provided the initial policy possesses sufficient
  diversity to cover the true answer.
\end{itemize}

C. Computational Properties.  The Arithmetic Mean is highly
tractable. For discrete samples, it is simply the empirical frequency
distribution:
\[
\sfq^*(z) = \frac{\text{Count}(z \in G)}{K}.
\]
This allows for a direct implementation in GRPO using a ``Soft
Consistency'' reward:
\[
R_i = \log \paren*{ \frac{\text{Count}(z_i)}{K} }.
\]
This reward is granular, distinguishing between strong consensus
(e.g., 90\%) and weak consensus (e.g., 40\%).

D. Statistical Properties.  The Arithmetic Mean enforces \emph{Union
  Consistency}: it preserves any answer generated by at least one
group member. However, it admits a blurring effect issue: a key
statistical downside is that averaging distributions \emph{increases}
entropy (by Jensen's inequality). While we want the model to sharpen
(become more confident), the Arithmetic Mean target effectively asks
the model to cover the uncertainty of the group. This blurring
must be counteracted by the inherent sharpening of the model's
generation process (e.g., $\argmax$ decoding).

\subsection{ Majority Vote (Hard Approximation)}

A. Mathematical Definition.  The Majority Vote (or Mode) is the Dirac
distribution centered at the most frequent answer in the group:
\[
  \sfq^*(z)
  = \Ind\bracket*{ z
    = \argmax_{z'} \paren*{ \frac{1}{K} \sum_{k=1}^K \Ind[z_k = z'] } }.
\]

B. Approximation Status.  This is a \emph{heuristic approximation}. It
can be viewed as computing the Arithmetic Mean (step 1) and then
applying a hard $\argmax$ operator (step 2) to force the distribution
to collapse. This makes it an even coarser approximation than the
Arithmetic Mean, discarding all information about secondary
candidates.

C. Computational Properties. This is \emph{trivially tractable} and is
the basis of our \emph{\algname{Coherence-GRPO}} algorithm. It yields
a binary reward signal ($1$ for the majority, $0$ otherwise), which
simplifies the optimization landscape but loses the nuance of the soft
probabilities used in the Arithmetic Mean.

D. Statistical Properties.  The Mode acts as a hard sharpener. It
explicitly solves the entropy problem of the Arithmetic Mean by
forcing low entropy. However, for small group sizes $K$, the sample
mode is a high-variance estimator. If the true distribution is
multimodal or flat, the sample mode may jump randomly between classes,
introducing noise into the reward signal.

\subsection{Stability analysis via generalized squared Hellinger distance}
\label{app:arithmetic_approximation}

While the Arithmetic Mean minimizes a different divergence than the
ideal Geometric Mean, we now prove that it serves as a valid
approximation with bounded error. We analyze the performance gap using
the \emph{Generalized Squared Hellinger (GSH) Distance}.

\ignore{
We analyze the performance gap between the theoretical two-step $\KL$
projection and the Hybrid Double-Step Projection (HDSP)
Algorithm~\ref{alg:hdsp}, where in the first projection step the
Euclidean Squared distance is used instead of $\KL$, that is the
arithmetic mean instead of the normalized geometric mean.
}

\begin{definition}[Generalized Squared Hellinger Distance]
  Let $\cS_x = \{\sfp_1, \dots, \sfp_K\}$ be the set of output probability
  distributions generated by $K$ sampled traces for an input $x$. The
  Generalized Squared Hellinger (GSH) Distance is defined as the
  defect in the normalization of the geometric mean:
\[
  H^2(\sfp_1, \ldots, \sfp_K)
  = 1 - \sum_{y \in \sY} \paren*{ \prod_{k=1}^K \sfp_k(y) }^{1/K}
\]
\end{definition}
Note that $H^2(\{\sfp_k\})$ is always in $[0, 1]$. GSH generalizes the
standard squared Hellinger distance to $K$ distributions by
considering the defect in the Generalized Bhattacharyya Coefficient
\citep{toussaint1974probability}.  It is a measure of diversity of
distributions. It is small when distributions are similar (low
diversity) and large when they differ (high diversity). It can be
bounded in terms of JS divergence, which is also a measure of
diversity.

Let the improvement of a solution $\pi$ be defined as
\[
 \Improv(\pi) = \E\bracket*{\sfD_F(\pi^*(x) \parallel \pi_0(x))} -
 \E\bracket*{\sfD_F(\pi^*(x) \parallel \pi(x))}.
\]

\begin{theorem}[Improvement Guarantee for Hybrid Projection]
\label{th:hybrid-stability-hellinger}
Let $\ov \pi_{GM}$ denote the normalized geometric mean of $\cS_x$ and
$\ov \pi_{AM}$ denote the arithmetic mean of $\cS_x$.  Let
$\h \pi_{Pure}$ and $\h \pi_{Hybrid}$ be the projections of
$\ov \pi_{GM}$ and $\ov \pi_{AM}$ onto $\Pi$ respectively.

Assume the evaluation function $\ell(\pi) = \KL(\pi^* \parallel \pi)$
is $L$-Lipschitz continuous with respect to the $L_2$
norm. Furthermore, assume the log-likelihood gradient of the model is
bounded such that $\|\nabla \log \pi(y)\|_2 \leq B$ for all $y$.
Then, the improvement degradation of the Hybrid algorithm is bounded by:
\[
  \Improv(\h \pi_{Hybrid})
  \geq \Improv(\h \pi_{Pure})
  - 2 L B \, \E_{x \sim \sD_\sX} \bracket*{ H^2(\{\sfp_k\}) }
\]
\end{theorem}

\begin{proof}
  We seek to bound
  $\Delta = \Improv(\h \pi_{Pure}) - \Improv(\h \pi_{Hybrid})$.  By
  the $L$-Lipschitzness assumption of the improvement function, we
  have
\begin{equation}
\Delta \leq L \, \E_x \bracket*{ \norm*{ \h \pi_{Hybrid}(x) - \h \pi_{Pure}(x) }_2 }.
\end{equation}
For any fixed distribution $\sfq$,
$\pi \mapsto \KL(\sfq \parallel \pi)$ is known to be strongly convex
with respect to $L_1$ (Pinsker's inequality) and therefore also $L_2$,
since norm-1 upper bounds norm-2. Since both $\h \pi_{Hybrid}$ and
$\h \pi_{Pure}$ are minimizers, by standard stability results for
strongly convex objectives, the parameter distance is bounded by
the difference in gradients:
\[
  \| \h \pi_{Hybrid} - \h \pi_{Pure} \|_2
  \leq \| \nabla \mathcal{L}(\ov \pi_{AM}) - \nabla \mathcal{L}(\ov \pi_{GM}) \|_2
\]
Since the loss $\mathcal{L}(q, \pi) = -\sum_y \sfq(y) \log \pi(y)$ is
linear in $\sfq$, and assuming bounded gradients
$\|\nabla \log \pi(y)\|_2 \leq B$, we have:
\begin{equation}
  \| \nabla \mathcal{L}(\ov \pi_{AM}) - \nabla \mathcal{L}(\ov \pi_{GM}) \|_2
  \leq B \cdot \| \ov \pi_{AM} - \ov \pi_{GM} \|_1.
\end{equation}
We now bound the $L_1$ distance using the triangle inequality. Let
$G(y) = \paren*{\prod_{k = 1}^K \sfp_k(y)}^{1/K}$ be the unnormalized
geometric mean and $Z = 1 - H^2(\sfp_1, \ldots, \sfp_K)$ be its
sum. Note that $\ov \pi_{GM} = G/Z$.  Using the triangle inequality
via $G$:
\begin{align*}
  \| \ov \pi_{AM} - \ov \pi_{GM} \|_1
  & \leq \| \ov \pi_{AM} - G \|_1 + \| G - \ov \pi_{GM} \|_1\\
  & = \sum (\ov \pi_{AM} - G) + \sum |G - G/Z|
  \tag{by the AM-GM inequality $\ov \pi_{AM} \geq G$ and $\ov \pi_{GM} = G/Z$}\\
  & = (1 - Z) + Z(1/Z - 1)
  \tag{$(1 - Z) \geq 0$ and $1/Z - 1 \geq 0$}\\
  & = 2(1 - Z) = 2 H^2(\{\sfp_k\}).
\end{align*}
Thus, we have
\[
\Delta \leq LB \, \E_x [ 2 H^2(\{\sfp_k\}) ].
\]
Rearranging terms completes the proof.
\end{proof}

\textbf{Interpretation and Significance.}
Theorem \ref{th:hybrid-stability-hellinger} establishes an improvement
guarantee for the \emph{hybrid} algorithm where, instead of the
geometric mean (solution when $\KL$ is used in the the first step of
the double-step self-imporvement via coherence) the arithmetic mean is
used (solution when the squared-distance is used as a Bregman
divergence).

The term $H^2(\{\sfp_k\})$ measures the disagreement within the
specific batch of $K$ outputs generated for input $x$: if the sampled
traces yield answer distributions that are not too dissimilar (low
empirical diversity), $H^2$ is close to $0$, and the improvement
achieved by the algorithm is close to the full theoretical
improvement.  Thus, the theorem shows that the algorithmic use of the
arithmetic mean is a valid approximation of the theoretical geometric
mean when the empirical diversity is not too large on average.
Crucially, the bound scales linearly with this sample diversity,
ensuring the method remains stable even when the sampled evidence is
moderately incoherent.

\textbf{Remarks on Assumptions.}
The theorem relies on two standard regularity conditions. First, the
assumption that the evaluation function is $L$-Lipschitz is necessary
to bound the impact of parameter deviations on the final loss; while
the gradient $\nabla \KL$ is unbounded at the boundary of the simplex,
this condition holds in practice for models using softmax
parameterization with bounded logits (e.g., via weight decay or
explicit clamping). Second, the assumption of bounded log-likelihood
gradients ($B$) is a standard smoothness property of neural networks,
ensuring that small shifts in the target distribution do not require
arbitrarily large parameter updates. Together, these conditions define
a standard \emph{well-behaved} regime where the stability of the
projection is guaranteed.

\section{Extended Background on Trace-Level Coherence}
\label{app:trace_level_coherence}

In this appendix, we detail the theoretical foundations of trace-level
self-improvement as established in \citet{MohriSchneiderWu2025}. This
framework serves as the precursor to the Answer-Level Coherence Game
introduced in Section 5.

\subsection{Problem Formulation}

Let $\pi: \mathcal{X} \to \Delta(\mathcal{Y})$ be a conditional policy
mapping inputs to distributions over reasoning traces. We assume an
\emph{invariance mapping} $\Phi: \mathcal{X} \to \mathcal{X}$
(typically an involution, $\Phi(\Phi(x)) = x$) that preserves the
semantic meaning of the input.  A policy is said to be
\textbf{coherent} if it is invariant under this mapping:
\[
\pi(y|x) = \pi(y|\Phi(x)) \quad \forall x \in \mathcal{X}, y \in \mathcal{Y}.
\]
We define $\mathcal{C}_{\text{coh}}$ as the set of all such coherent
policies. The goal of self-improvement is to project the initial
policy $\pi_0$ onto $\mathcal{C}_{\text{coh}}$ using a Bregman
divergence $\sfD_F$ (e.g., KL divergence or Squared Euclidean distance).

\subsection{Projection Mechanisms}

\citet{MohriSchneiderWu2025} analyze two mechanisms for solving this
problem:

\begin{enumerate}
\item \textbf{Direct Projection:} The policy is projected directly
  onto the intersection of coherent models and the feasible model
  class $\Pi$:
    \[
      \h \pi
      = \argmin_{\pi \in \Pi \cap \mathcal{C}_{\text{coh}}}
      \E_x [ \sfD_F(\pi(\cdot|x) \| \pi_0(\cdot|x)) ].
    \]
    
  \item \textbf{Two-Step Projection:} This relaxes the feasibility
    constraint.
    \begin{itemize}
    \item \emph{Step 1 (Consensus):} Project $\pi_0$ onto the
      unconstrained space of coherent distributions
      $\mathcal{C}_{\text{coh}}^\dagger$. For KL divergence, this
      corresponds to computing the \emph{Geometric Mean} of the
      distributions in the orbit
      $\{ \pi_0(\cdot|x), \pi_0(\cdot|\Phi(x)) \}$.
    \item \emph{Step 2 (Distillation):} Project the consensus
      distribution back onto the valid model space $\Pi$ (typically
      via supervised fine-tuning).
    \end{itemize}
\end{enumerate}
A key result (Theorem 5.7 in \citet{MohriSchneiderWu2025}) establishes
that for a broad class of divergences (including KL and Euclidean),
these two mechanisms are equivalent: the global projection coincides
with the two-step procedure.

\subsection{Theoretical Guarantees}

The central guarantee of this framework is the \emph{Pythagorean
  Improvement Theorem}. If the ideal ground-truth policy $\pi^*$ is
coherent (i.e., $\pi^* \in \mathcal{C}_{\text{coh}}$), then the
projected policy $\h \pi$ satisfies:
\[
  \E_x [ \sfD_F(\pi^* \| \h \pi) ] \le \E_x [ \sfD_F(\pi^* \| \pi_0) ]
  - \E_x [ \sfD_F(\h \pi \| \pi_0) ].
\]
This inequality ensures that the projection strictly reduces the
distance to the optimum. Intuitively, by eliminating the "incoherent
noise" component of the error (the distance
$\sfD_F(\h \pi \| \pi_0)$), the model necessarily moves closer to the
true distribution.

%% file: omitted.tex
\section{Omitted Proofs}\label{sec:omitted}

\subsection{Proof of Theorem~\ref{thm:consistency}}

\begin{proof}[Proof of Theorem~\ref{thm:consistency}]
  Equivalence: This follows directly from the Fenchel Duality
  Theorem. Since $\cR$ is convex and l.s.c.,
  $\cR(\nu_\pi) = \sup_\sfq (-\beta \langle \nu_\pi, \log
  \sfq \rangle - \Psi(\sfq))$. Substituting this into the
  definition of $\cJ(\pi)$ yields
  $\min_\pi \max_\sfq \cG(\pi, \sfq)$.

  Optimal Policy Form: Consider minimizing
  $\cG(\pi, \sfq)$ with respect to $\pi$ for a fixed
  $\sfq$. The relevant terms are:
\[
  \min_{\pi \in \Pi} \E_x \bracket*{ \beta \E_{y \sim \pi} \bracket*{\log \frac{\pi(y)}{\pi_0(y)}}
    - \beta \E_{y \sim \pi} [\log \sfq(\EX(y))] }
= \min_{\pi \in \Pi} \beta \E_x \bracket*{\E_{y \sim \pi}
\bracket*{ \log \frac{\pi(y)}{\pi_0(y) \sfq(\EX(y))} }}.
\]
This is equivalent to minimizing $\KL(\pi \parallel \wt \sfq)$,
where $\wt \sfq(y) \propto \pi_0(y)\sfq(\EX(y))$.
\end{proof}

%% file: alft.bib
@book{Rockafellar1996,
Author = {Rockafellar, R. Tyrrell},
Title = {Convex analysis},
Publisher = {Princeton University Press},
Year = {1997}}

@article{toussaint1974probability,
  title={Probability of error, expected divergence, and the affinity of several distributions},
  author={Toussaint, Godfried T},
  journal={IEEE Transactions on Systems, Man, and Cybernetics},
  volume={SMC-4},
  number={5},
  pages={485--488},
  year={1974},
  publisher={IEEE}
}

@article{MohriSchneiderWu2025,
  author       = {Mehryar Mohri and
                  Jon Schneider and
                  Yifan Wu},
  title        = {Coherence Mechanisms for Provable Self-Improvement},
  journal      = {CoRR},
  volume       = {abs/2511.08440},
  year         = {2025},
  url          = {https://doi.org/10.48550/arXiv.2511.08440},
  doi          = {10.48550/ARXIV.2511.08440},
  eprinttype    = {arXiv},
  eprint       = {2511.08440},
  bibsource    = {dblp computer science bibliography, https://dblp.org}
}

@article{FreundSchapire1999,
  title={Adaptive game playing using multiplicative weights},
  author={Freund, Yoav and Schapire, Robert E},
  journal={Games and Economic Behavior},
  volume={29},
  number={1-2},
  pages={79--103},
  year={1999},
  publisher={Elsevier}
}

@article{ShalevShwartz2012,
  title={Online learning and online convex optimization},
  author={Shalev-Shwartz, Shai},
  journal={Foundations and Trends{\textregistered} in Machine Learning},
  volume={4},
  number={2},
  pages={107--194},
  year={2012},
  publisher={Now Publishers, Inc.}
}

@inproceedings{Wei2022,
  title={Chain-of-Thought Prompting Elicits Reasoning in Large Language Models},
  author={Wei, Jason and Wang, Xuezhi and Schuurmans, Dale and Bosma, Maarten and Xia, Fei and Chi, Ed and Le, Quoc V and Zhou, Denny},
  booktitle={Advances in Neural Information Processing Systems (NeurIPS)},
  volume={35},
  pages={24824--24837},
  year={2022}
}

@article{Lightman2023,
  title={Let's Verify Step by Step},
  author={Lightman, Hunter and Kosaraju, Vineet and Burda, Yura and Edwards, Harri and Baker, Bowen and Lee, Teddy and Leike, Jan and Schulman, John and Sutskever, Ilya and Cobbe, Karl},
  journal={arXiv preprint arXiv:2305.20050},
  year={2023}
}

@inproceedings{Wang2022,
  title={Self-Consistency Improves Chain of Thought Reasoning in Language Models},
  author={Wang, Xuezhi and Wei, Jason and Schuurmans, Dale and Le, Quoc V and Chi, Ed H and Narang, Sharan and Chowdhery, Aakanksha and Zhou, Denny},
  booktitle={The Eleventh International Conference on Learning Representations (ICLR)},
  year={2023},
  note={Preprint arXiv:2203.11171 (2022)}
}

@inproceedings{Rafailov2023,
  title={Direct Preference Optimization: Your Language Model is Secretly a Reward Model},
  author={Rafailov, Rafael and Sharma, Archit and Mitchell, Eric and Manning, Christopher D and Ermon, Stefano and Finn, Chelsea},
  booktitle={Advances in Neural Information Processing Systems (NeurIPS)},
  volume={36},
  pages={53728--53741},
  year={2023}
}

@article{Li2025,
  title={Jointly Reinforcing Diversity and Quality in Language Model Generations},
  author={Li, Tian and Zhang, Yifan and Yu, P and Saha, S and Khashabi, D and Weston, J and Lanchantin, J and Wang, T},
  journal={arXiv preprint arXiv:2509.02534},
  year={2025},
  url={https://arxiv.org/abs/2509.02534}
}

@article{Shao2024,
  title={{DeepSeekMath}: Pushing the Limits of Mathematical Reasoning in Open Language Models},
  author={Shao, Zhihong and Wang, Peiyi and Zhu, Qihao and Xu, Runxin and Song, Junxiao and Zhang, Mingchuan and Li, Y.K. and Wu, Y. and Guo, Daya},
  journal={arXiv preprint arXiv:2402.03300},
  year={2024}
}

@book{CoverThomas2006,
  title={Elements of Information Theory},
  author={Cover, Thomas M. and Thomas, Joy A.},
  year={2006},
  publisher={Wiley-Interscience},
  edition={2nd},
  address={Hoboken, NJ}
}

@article{Tseng2001,
  title={Convergence of a Block Coordinate Descent Method for Nondifferentiable Minimization},
  author={Tseng, Paul},
  journal={Journal of Optimization Theory and Applications},
  volume={109},
  number={3},
  pages={475--494},
  year={2001},
  publisher={Springer}
}

@article{Uesato2022,
  title={Solving math word problems with process- and outcome-based feedback},
  author={Uesato, Jonathan and Kushman, Nate and Kumar, Ramana and Song, Francis and Siegel, Noah and Wang, Lisa and Creswell, Antonia and Irving, Geoffrey and Higgins, Irina},
  journal={arXiv preprint arXiv:2211.14275},
  year={2022}
}
